%% file: paper.tex
\documentclass[twoside,leqno,twocolumn]{article}
\usepackage{ltexpprt}

\usepackage{amsfonts}
\usepackage{amsmath}
\usepackage{hyperref}
\usepackage{units}
\usepackage{subcaption}
\usepackage{booktabs}
\usepackage{graphicx}
\usepackage{adjustbox}

\usepackage{bold-extra}

\usepackage[algo2e]{algorithm2e}
\usepackage{float}
\usepackage{placeins}

\usepackage[title]{appendix}

\newtheorem{problem}{Problem}[section]
\DeclareMathOperator*{\argmax}{arg\,max}

\newcommand*{\qed}{\hfill\ensuremath{\square}}%

\begin{document}

\title{\Large Human-guided Data Exploration Using Randomisation\thanks{Supported by the Academy of Finland (decisions 319145 and 313513).}}
\author{
  Kai Puolam\"aki\thanks{Department of Computer Science, University of Helsinki, Finland. Email: firstname.lastname@helsinki.fi}  \hspace*{2em} Emilia Oikarinen$^\dag$ \hspace*{2em} Buse
  Atli\thanks{Department of Computer Science, Aalto University, Finland.  Email: firstname.lastname@aalto.fi} \hspace*{2em} Andreas Henelius$^\dag$ \\
 \\
}
\date{}

\maketitle

\fancyfoot[R]{\scriptsize{Copyright \textcopyright\ 2018\\
Copyright for this paper is retained by authors}}

\begin{abstract} \small\baselineskip=9pt An explorative data analysis
  system should be aware of what the user already knows and what the
  user wants to know of the data: otherwise the system cannot provide
  the user with the most informative and useful views of the data.  We
  propose a principled way to do exploratory data analysis, where the
  user's background knowledge is modeled by a distribution
  parametrised by subsets of rows and columns in the data, called
  tiles. The user can also use tiles to describe his or her interests
  concerning relations in the data. We provide a computationally
  efficient implementation of this concept based on constrained
  randomisation. The implementation is used to model both the
  background knowledge and the user's information request and is a
  necessary prerequisite for any interactive system. Furthermore, we
  describe a novel linear projection pursuit method to find and show
  the views most informative to the user, which at the limit of no
  background knowledge and with generic objectives reduces to PCA. We
  show that our method is robust under noise and fast enough for
  interactive use. We also show that the method gives understandable
  and useful results when analysing real-world data sets. We will
  release an open source library implementing the idea, including the
  experiments presented in this paper.  We show that our method can
  outperform standard projection pursuit visualisation methods in
  exploration tasks.  Our framework makes it possible to construct
  human-guided data exploration systems which are fast, powerful, and
  give results that are easy to comprehend.

\end{abstract}

\input{manuscript}
\bibliographystyle{abbrv}
\bibliography{paper}
\begin{appendices}

\input{appendix}
\end{appendices}

\end{document}

%% file: manuscript.tex
\section{Introduction}
Exploratory data analysis \cite{tukey:1977}, often performed
interactively, is an established technique for learning about
relations in a dataset prior to more formal analyses. Humans can
easily identify patterns that are relevant for the task at hand but
often difficult to model algorithmically. Current visual exploration
systems lack a principled approach for this process. \textbf{Our goal
  and main contribution} in this paper is to \emph{devise a framework
  for human-guided data exploration by modeling the user's background
  knowledge and objectives and using these to offer the user the most
  informative views of the data.}.
\begin{figure*}
\begin{footnotesize}
\begin{tabular}{ccc}
\includegraphics[width=0.3\textwidth]{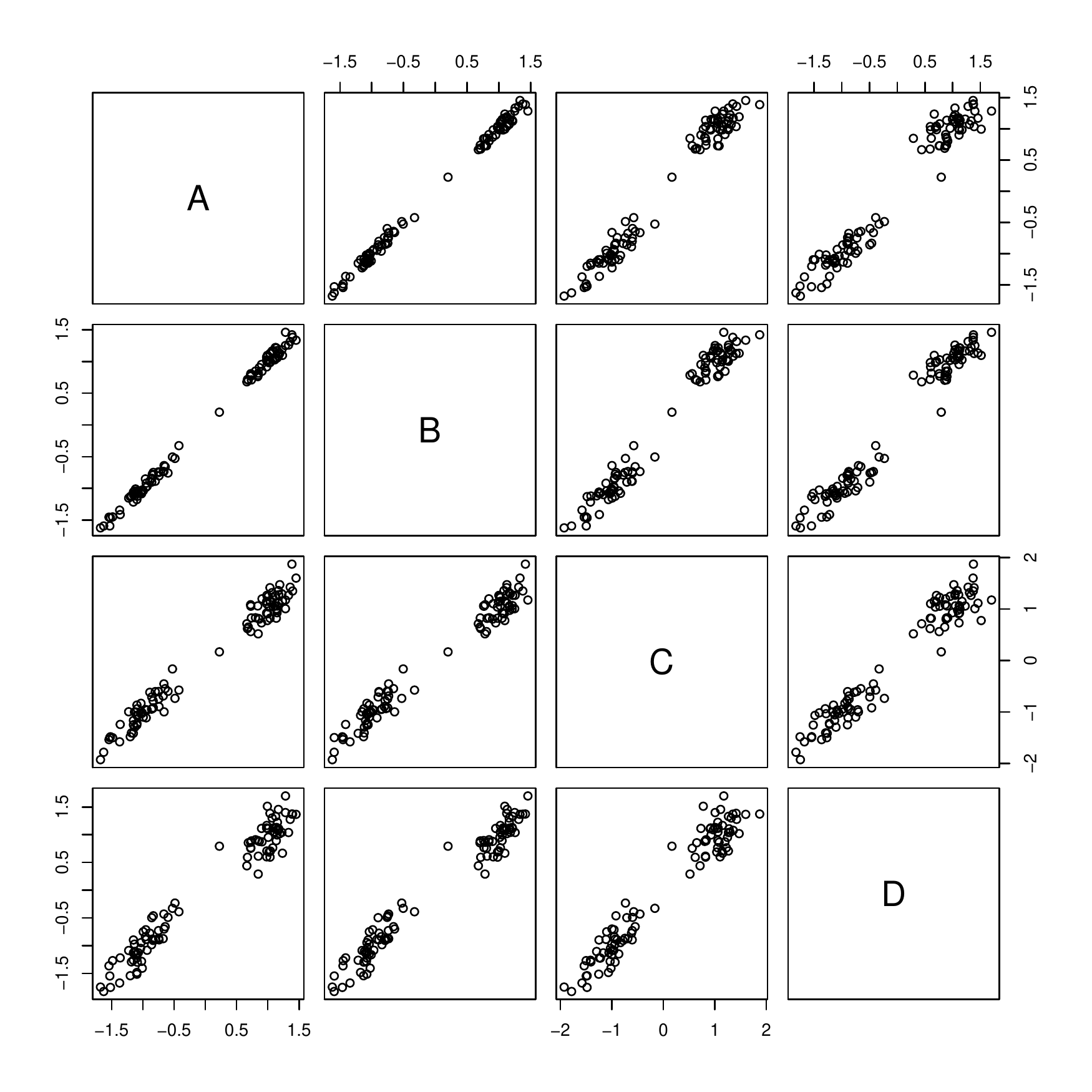}&
\includegraphics[width=0.3\textwidth]{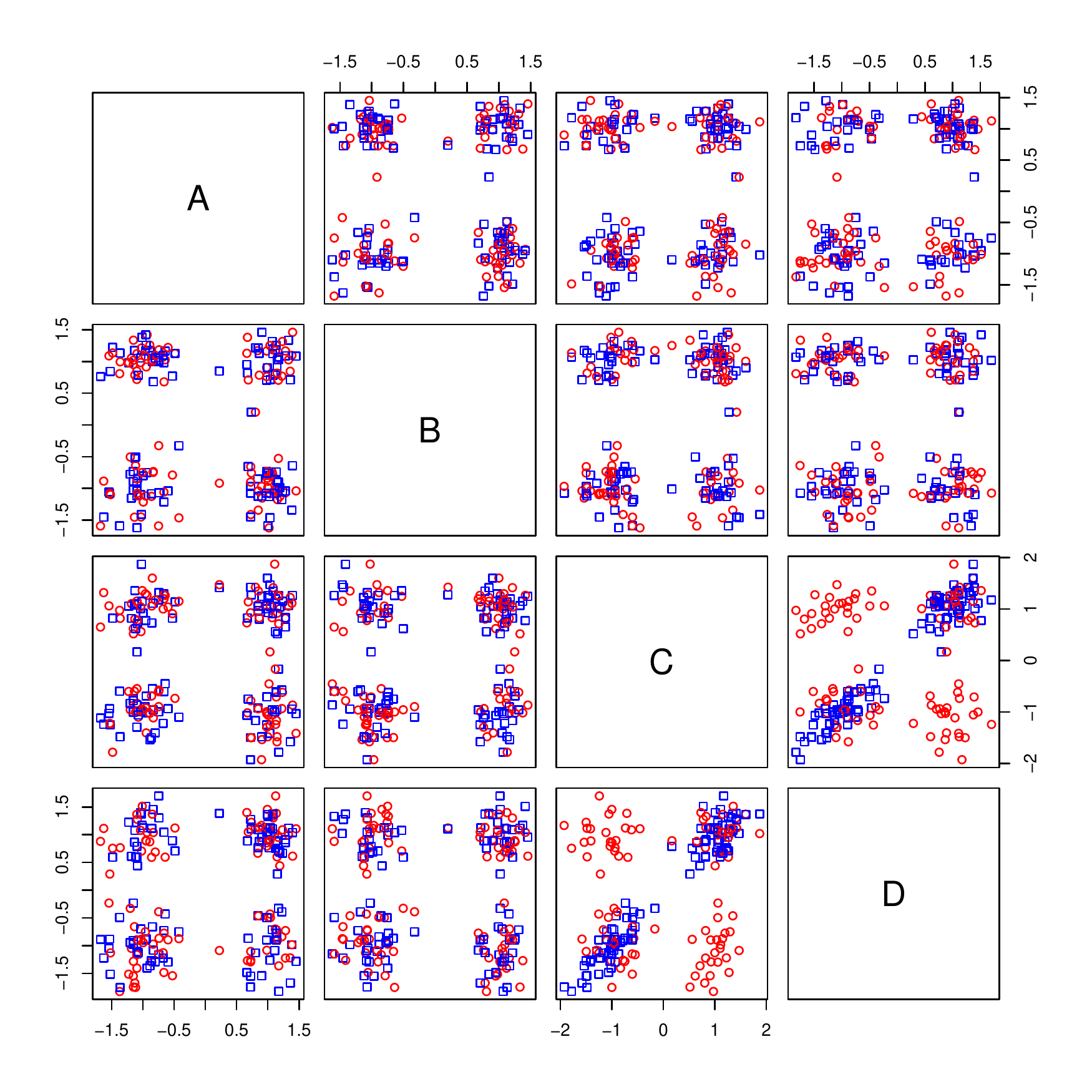}&
\includegraphics[width=0.3\textwidth]{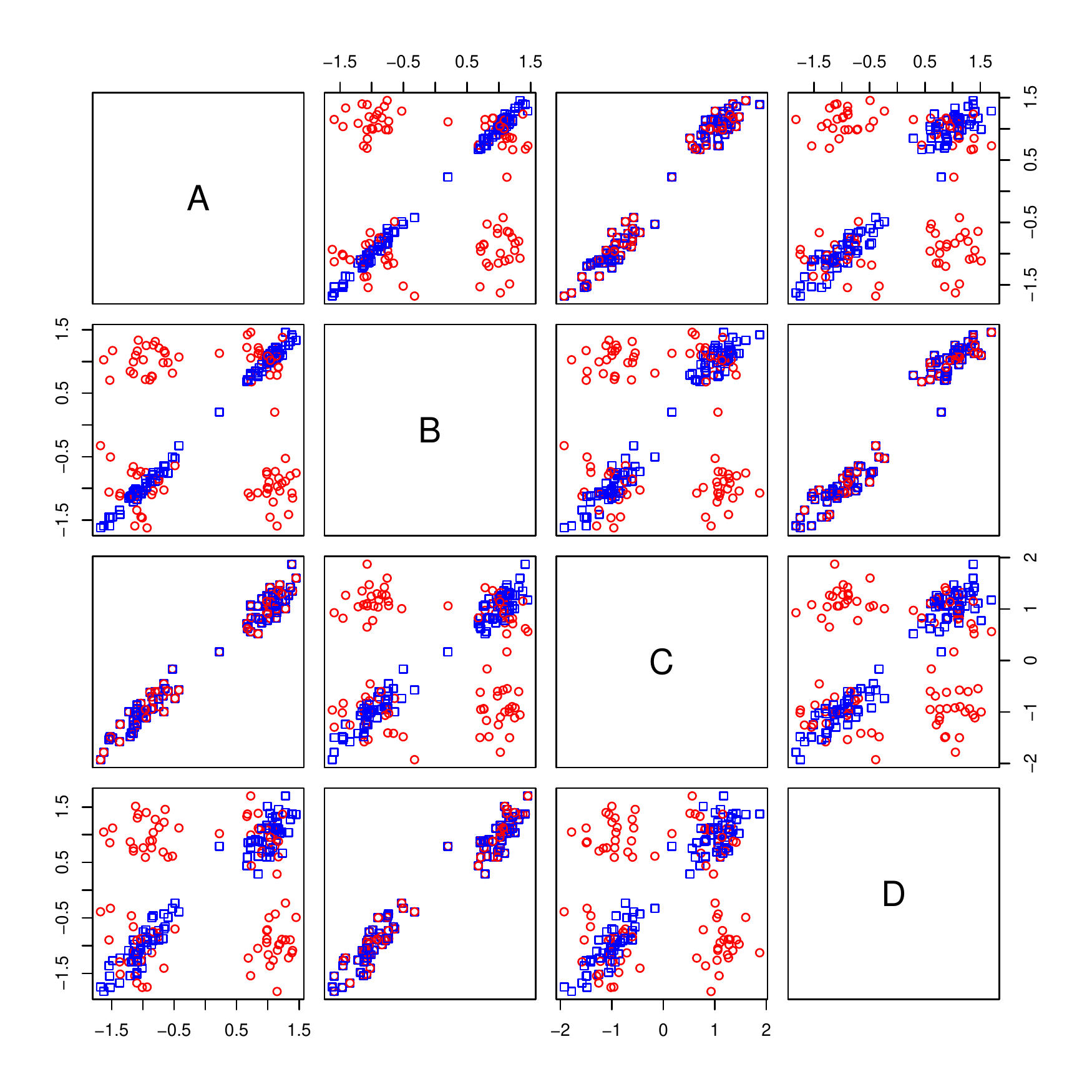}\\
 (a)&(b)&(c)\\
\end{tabular}
\end{footnotesize}
\caption{\label{fig:toy} (a) Pairplot matrix of toy data. (b) Fully
  randomised data (red spheres) that models user's knowledge and data
  where only relation between attributes $CD$ has been preserved (blue
  squares) that models what the user could learn of the relations
  between attributes $CD$. (c) As
  in (b), but additionally relation between attributes $AC$, as well
  as between attributes $BD$, has been preserved, modeling user's
  knowledge of the relations between $AC$ and $BD$, respectively.}
\end{figure*}

Our contribution consists of two main parts: (i) a framework to model
and incorporate the user's background knowledge of the data and to
express the user's objectives, and (ii) a system to show the most
informative views of the data. The first contribution is general, but
the second contribution, the most informative views of the data, is
specific to a particular data type. In this paper we focus on data
items that can be represented as real-valued vectors of attribute
values.

As an example, consider the 4-dimensional toy data set shown in
Fig.~\ref{fig:toy}a. The dataset has been generated by first creating
the strongly correlated attributes $A$ and $B$, and then generating
attribute $C$ by adding noise to $A$, and attribute $D$ by adding
noise to $B$. The purpose of this example is to show how the user's
background knowledge and objectives affect the views that are most
informative to the user.

Assume the {\em user is interested in the relation of attributes $C$
  and $D$}; we call the relation of interest a {\em hypothesis}. Our
task is to find a maximally informative 1-dimensional projection of
the data that takes both this objective and the user's background
knowledge into account.\footnote{It is not usually possible for the
  user to view the whole distribution at once, hence, it is necessary,
  e.g., to view projections of the data.}

First, assume that the user knows only the marginal distributions of
the attributes but nothing of their relations. We argue that in such a
case the user's internal model of the data can be modeled by a
distribution over data sets that we call the {\em background
  distribution}, which in this case can be sampled from by permuting
the columns of the data matrix at random, shown by the red spheres in
Fig. \ref{fig:toy}b. Because we are interested in the relation of $CD$
we create another distribution in which $CD$ are permuted together but
the data is otherwise randomised, as shown by the blue squares in
Fig. \ref{fig:toy}b. The red distribution models what the user knows
and the blue what the user could optimally learn about the relation of
$CD$ from the data. The red and blue distributions differ most in the
plot $CD$---as one would expect!---and indeed the maximally
informative 1-dimensional projection is given by $C+D$.\footnote{See
  Eq. \eqref{eq:gain} for a formal definition of informativeness used
  in this paper.}

Secondly, assume that---unlike above---the user already knows the
relationship of attributes $AC$ and $BD$, but does not yet know that
attributes $AB$ are almost identical. We can repeat the previous
exercise with the difference that we now add the user's knowledge as
constraints to {\em both} the red {\em and} blue distributions, i.e.,
we permute $AC$ and $BD$ together (modeling the user's knowledge of
these relations), as shown in Fig. \ref{fig:toy}c. Again, the red
distribution models the user's knowledge and the blue what the user
could learn from the relation of $CD$ from the data, given that the
user already knows the relationships of $AC$ and $BD$. The red and
blue distributions differ most in plot $AB$ and therefore the user
would gain most information if shown this view; indeed, the most
informative projection is $A+B$. In other words, the knowledge of the
relation of $A$ and $B$ gives maximal information about the relation
of $CD$! This makes sense, because variables $CD$ are really connected
via $AB$ through their generative process.

{\bf Background model and objectives.} We model the user's background
knowledge by a distribution over all datasets. We define the
distribution using permutations and constraints on the permutations,
which we call {\em tiles}. A tile is defined by a subset of data items
and attributes. All attributes within a tile are permuted with the
same permutation to conserve the relation between attributes. When the
user has not seen the data, we assume that the background distribution
is unconstrained and can be sampled from by permuting each attribute
of the data at random. The user can input observations of the data
using tiles. Essentially, by constructing a tile the user acknowledges
that he or she knows the relations within the tile.

The tiles can also be used to formalise the user's objectives. For example, if the user is interested in the interaction
between two groups of variables, he or she can define two
distributions using tiles, which we call {\em hypotheses}: one in
which the interaction of interest is preserved and one in which it is
broken. Any difference between these two hypotheses gives the user new
information about the interaction of interest.

{\bf Finding views.} The data and the hypotheses are typically
high-dimensional and it is in practice not possible to view all
relations at once; if it were, the whole problem would be trivial by
just showing the user the entire dataset in one view. As a
consequence, we need to construct a visualisation or a dimensionality
reduction method that shows the most informative view (defined in
Sec.~\ref{sec:methods}) of the differences between the hypotheses. We
introduce a linear projection pursuit method that finds a projection
in a direction in which the two hypotheses differ most. The proposed
method seeks directions in which the ratio of the variance of these
two distributions is maximal. At the limit of no background
information and most generic hypotheses the method reduces to PCA
(Thm.~\ref{thm:pca}).

The domain of interactive data exploration sets some further {\bf
  requirements} for any implementation. On one hand, our system has no
need of scaling to a huge number of data points, since visualising an
extremely large number of points makes no sense. In practice, if the
number of data points is large enough, we can always downsample the
data to a manageable size. Therefore, our system has essentially a
constant time complexity with respect to the number of data items, but
not to the number of attributes, as shown later in
Sec.~\ref{sec:experiments}. On the other hand, the response times must
be on the order of seconds for fluid interaction. This rules out many
slow to compute but otherwise sound approaches.

In summary, the {\bf contributions} of this paper are: (i) a
computationally efficient formulation and implementation of the user's
background model and hypotheses using constrained randomisation, (ii)
a dimensionality reduction method to show the most informative view to
the user, and (iii) an experimental evaluation that supports that our
approach is fast, robust, and produces easily understandable results.
The Appendix contains an algorithm for merging tiles
(Sec.~\ref{sec:merge}), and an example demonstrating exploration of
the German data (Secs.~\ref{sec:german1} and \ref{sec:german2})

\section{Methods}\label{sec:methods}

Let $X$ be an $n \times m$ data matrix (dataset). Here $X(i,j)$
denotes the $i$th element in column $j$. Each column $X(\cdot, j),\ j
\in [m]$, is an \emph{attribute} in the dataset, where we used the
shorthand $[m] = \{1, \ldots, m\}$. Let $D$ be a finite set of domains
(e.g., continuous or categorical) and let $D(j)$ denote the domain of
$X(\cdot, j)$. Also let $X(i, j) \in D(j) \textrm{ for all } i \in
[n],\ j \in [m]$, i.e., all elements in a column belong to the same
domain, but different columns can have different domains. The
derivation in Sec.~\ref{sec:bg} is generic, but in Sec.~\ref{sec:view}
we consider only real numbers, $D(j) \subseteq \mathbb{R}$ for all
$j\in[m]$.

\subsection{Background model and tile constraints}\label{sec:bg}
In this subsection, we introduce the permutation-based sampling method
and tiles which can be used to constrain the sampled distribution and
to express the user's background knowledge and objectives
(hypotheses). The sampled distribution is constructed so that in the
absence of constraints (tiles) the marginal distributions of the
attributes are preserved.

We define a \emph{permutation} of the data matrix $X$ as follows.
\begin{Definition}[Permutation]\label{def:permutation}{
  Let ${\mathcal P}$ denote the set of permutation functions of length
  $n$ such that \\
  $\pi:[n]\mapsto[n]$ is a bijection for all
  $\pi\in{\mathcal P}$, and denote by
  $\left(\pi_1,\ldots,\pi_m\right)\in{\mathcal P}^m$ the vector of
  column-specific permutations. A permutation $\widehat X$ of the data
  matrix $X$ is then given as $\widehat X(i,j)=X(\pi_j(i),j)$.}
\end{Definition}
When permutation functions are sampled uniformly at random, we obtain
a uniform sample from the distribution of
datasets where each of the attributes has the same marginal
distribution as the original data. We parametrise this distribution with
\emph{tiles} that preserve the relations in the data matrix $X$ for a
subset of rows and columns: a tile is a tuple $t = (R, C)$, where $R
\subseteq [n]$ and $C \subseteq [m]$.  The tiles considered here are
combinatorial (in contrast to geometric), meaning that rows and
columns in the tile do not need to be consecutive.  In an
unconstrained case, there are $(n!)^m$ allowed vectors of
permutations. The tiles constrain the set of allowed permutations as
follows.
\begin{Definition}[Tile constraint]
  \label{def:tileconstraint}
  Given a tile \\
  $t = (R,C)$, the vector of permutations
  $\left(\pi_1,\ldots,\pi_m\right)\in{\mathcal P}^m$ is allowed by $t$
  iff the following condition is true for all $i\in[n]$, $j\in[m]$,
  and $j'\in[m]$:
$$
i\in R\wedge j,j'\in C  \\
\implies \pi_j(i)\in R\wedge \pi_j(i)=\pi_{j'}(i)
.$$
Given a set of tiles $T$, a set of
permutations is allowed iff it is allowed by all $t\in T$.
\end{Definition}
A tile defines a subset of rows and columns, and the rows in this
subset are permuted by the same permutation function in each column in
the tile. In other words, the \emph{relations between the columns
  inside the tile are preserved}.  Notice
that the identity permutation is always an allowed permutation.
Now, the sampling problem can be formulated as follows.
\begin{problem}[Sampling problem]
\label{def:samplingproblem}
Given a set of tiles $T$,  draw samples
uniformly at random from vectors of permutations in ${\mathcal P}^m$ allowed by $T$.
\end{problem}
The sampling problem is trivial when the tiles are non-overlapping,
since permutations can be done independently within each
non-overlapping tile. However, in the case of overlapping tiles,
multiple constraints can affect the permutation of the same subset of
rows and columns and this issue must be resolved. To this end, we need
to define the equivalence of two sets of tiles, which means that the
same constraints are enforced on the permutations.
\begin{Definition}[Equivalence of sets of tiles]
\label{def:equivalence}
Let $T$ and $T'$ be two sets of tiles. $T$ is \emph{equivalent} to
$T'$, if $P$ is allowed by $T$ iff $P$ is allowed by $T'$ for all
vectors of permutations $P\subseteq \mathcal{P}^m$.
\end{Definition}

We say that a set of tiles $\mathcal{T}$ where no tiles overlap, is a
\emph{tiling}.  Next, we show that there always exists a tiling
equivalent to a set of tiles.
\begin{theorem}
Given a set of (possibly overlapping) tiles $T$, there exists a tiling $\mathcal{T}$ that
 is equivalent to $T$.
\end{theorem}
\begin{proof}
Let $t_1 = (R_1, C_1)$ and $t_2 = (R_2, C_2)$ be two overlapping
tiles.
Each tile describes a set of
constraints on the allowed permutations of the rows in their
respective column sets $C_1$ and $C_2$. A  tiling $\{t'_1, t'_2, t'_3\}$  equivalent to $\{t_1,t_2\}$ is given by:
$$\begin{array}{rcl}
  t'_1& = &(R_1\setminus R_2, C_1), \\
  t'_2 &= &(R_1\cap R_2, C_1\cup C_2), \\
  t'_3 &= &(R_2\setminus R_1, C_2).
\end{array}$$
\noindent Tiles $t'_1$ and $t'_3$
represent the non-overlapping parts of $t_1$ and $t_2$ and the
permutation constraints by these parts can be directly met. Tile
$t'_2$ takes into account the combined effect of $t_1$ and $t_2$ on
their intersecting row set, in which case the same permutation
constraints must apply to the union of their column sets. It follows
that these three tiles are non-overlapping and enforce the combined
constraints of tiles $t_1$ and $t_2$. Hence, a tiling can be
constructed by iteratively resolving overlap in a set of tiles until
no tiles overlap. \qed
\end{proof}
Notice that merging overlapping tiles leads to
wider (larger column set) and lower (smaller row set) tiles. The
limiting case is a fully-constrained situation where each row is a
separate tile and only the identity permutation is allowed. We provide an efficient algorithm for merging tiles in Appendix \ref{sec:merge}.

\subsection{Formulating hypotheses}
\label{sec:formulatinghypotheses}
Our goal is to compare two distributions and we constrain the
distributions in question by forming hypotheses. Tilings are used to
form the hypotheses. The so-called {\em hypothesis tilings} provide a
flexible method for the user to specify the relations in which he or
she is interested.
\begin{Definition}[Hypothesis tilings]\label{def:hypothesis}
Given a subset of rows $R\subseteq[n]$, a subset of columns
$C\subseteq[m]$, and a $k$-partition of the columns given by
$C_1,\ldots,C_k$, such that $C=\cup_{i=1}^k{C_k}$ and $C_i\cap
C_j=\emptyset$ if $i\ne j$, a pair of \emph{hypothesis tilings} is
given by $\mathcal{T}_{H_1}=\{(R,C)\}$ and
$\mathcal{T}_{H_2}=\cup_{i=1}^k{\{(R,C_i)\}}$, respectively.
\end{Definition}
The hypothesis tilings define the items $R$ and attributes $C$ of
interest and the relations between the attributes that the user is
interested in (through the partition of $C$).  \textsc{Hypothesis 1}
($\mathcal{T}_{H_1}$) corresponds to a hypothesis where all relations
in $(R,C)$ are preserved, and \textsc{hypothesis 2}
($\mathcal{T}_{H_2}$) to a hypothesis where there are no unknown
relations between attributes in the partitions $C_1,\ldots,C_k$ of
$C$.

For example, if the columns are partitioned into two groups $C_1$ and
$C_2$ the user is interested in relations \emph{between} the
attributes in $C_1$ and $C_2$, but not in relations \emph{within}
$C_1$ or $C_2$. On the other hand, if the partition is full, i.e.,
$k=|C|$ and $|C_i|=1$ for all $i\in [k]$, then the user is interested
in \emph{all} relations between the attributes. In the latter case,
the special case of $R=[n]$ and $C=[m]$ indeed reduces to {\bf unguided
data exploration}, where the user has no background knowledge and the
hypothesis covers all inter-attribute relations in the data.

The user's knowledge concerning relations in the data is described by
tiles as well. As the user views the data she or he can highlight
relations observed by tiles. For example, the user can mark an
observed cluster structure with a tile involving the data points in
the cluster and the relevant attributes. We denote the set of
user-defined tiles by ${\mathcal T}_u$. In our general framework, the
user compares
two distributions characterised by the tilings
$\mathcal{T}_u + \mathcal{T}_{H_1}$ and
$\mathcal{T}_u + \mathcal{T}_{H_2}$, respectively.
Here `$+$' is used with a slight abuse of notation to denote the
operation of merging tilings into an equivalent tiling. By
$\mathcal{H}$ we denote the pair of hypotheses
\begin{equation}\mathcal{H} =\langle \mathcal{T}_u + \mathcal{T}_{H_1}, \mathcal{T}_u
+ \mathcal{T}_{H_2}\rangle.
\end{equation}
Note that $\mathcal{H}$ specifies two distributions over datasets,
both parametrised by their respective tilings, from which we can draw
samples as described in Sec.~\ref{sec:bg}.

\subsection{Finding views}\label{sec:view}
We are now ready to formulate our second main problem, i.e., given two
distributions characterised by the hypothesis pair $\mathcal{H}$, how
can we find an {\em informative view} of the data maximally contrasting
these two distributions. The answer to this question
depends on the type of data and the visualisation selected. For example,
the visualisations or measures of difference are different for
categorical and real-valued data. The {\em real-valued data} discussed
in this paper allows us to use projections (such as principal
components) that mix attributes.
\begin{problem}[Comparing hypotheses]
\label{prob:comparinghypotheses}
Given two distributions characterised by the pair
$\mathcal{H} =\langle \mathcal{T}_u + \mathcal{T}_{H_1}, \mathcal{T}_u
+ \mathcal{T}_{H_2}\rangle$,
where $\mathcal{T}_u $ is a (user-defined) background model tiling,
and $\mathcal{T}_{H_1}$ and $\mathcal{T}_{H_2}$ are hypothesis
tilings, find the projection in which the distributions differ the
most.
\end{problem}

To formalise the optimisation criterion in
Prob. \ref{prob:comparinghypotheses},
we define a {\em gain function}
$G(u,\mathcal{H})$ by
\begin{equation}\label{eq:g}G(u,\mathcal{H})=u^T\Sigma_1u/u^T\Sigma_2u,\end{equation}
where $\Sigma_1$ and $\Sigma_2$ are the covariance matrices of
the distributions parametrised by
the tilings $ \mathcal{T}_u + \mathcal{T}_{H_1}$ and
$\mathcal{T}_u + \mathcal{T}_{H_2}$, respectively,
and $u$ is the projection direction. The covariance
matrices $\Sigma_1$ and $\Sigma_2$
can be found analytically by the following theorem.
\begin{theorem}
Given
$j,j'\in[m]$, the covariance of attributes $\mathrm{cov}(j,j')$
under the
distribution defined by the tiling $\mathcal{T}$ is given by
$\mathrm{cov}(j,j')=\sum\nolimits_{i=1}^n{a_j(i)a_{j'}(i)}/n$, where
\begin{equation*}
a_l(i)=\left\{
\begin{array}{ll}
Y(i,l),                                      &  i \in  R_{j,j'} \\
\sum\nolimits_{k\in R(i,l)}{Y(k,l)}/|R(i,l)|,  & i \notin  R_{j,j'}
\end{array}
\right.
\end{equation*}
and $l\in\{j,j'\}$.
We denote by
$R_{j,j'}=\{i\in[n] \mid \exists (R,C)\in\mathcal{T}\ldotp i\in R \wedge
j,j'\in C\}$
the set of rows permuted together, by
$Y(i,l)=X(i,l)-\sum\nolimits_{i=1}^n{X(i,l)}/n$ the centered data
matrix, and by $R(i,l)\subseteq[n]$ a set satisfying
$\exists C\subseteq[m]\ldotp(R(i,l),C)\in\mathcal{T}\wedge i\in R(i,l)\wedge  l\in C$, i.e.,
the rows in a tile that data point $X(i,l)$ belongs to.
\end{theorem}
\begin{proof}
  The covariance is defined by
  $$\mathrm{cov}(j,j')=
  E\left[\sum\nolimits_{i=1}^n{Y(\pi_j(i),j)Y(\pi_{j'}(i),j')}/n \right],$$
  where the expectation $E[\cdot]$ is defined over the
  permutations $\pi_j\in{\cal P}$ and $\pi_{j'}\in{\cal P}$ of columns
  $j$ and $j'$ allowed by the tiling, respectively. The part of the sum for rows permuted
  together $R_{j,j'}$ reads
  \begin{multline*}
  \sum\nolimits_{i\in R_{j,j'}}{E\left[Y(\pi_j(i),j)Y(\pi_{j'}(i),j')\right]}/n \\=
  \sum\nolimits_{i\in R_{j,j'}}{Y(i,j)Y(i,j')}/n,
  \end{multline*}
where we have used $\pi_j(i)=\pi_{j'}(i)$ and reordered the sum for
$i\in R_{j,j'}$. The remainder of the sum reads
   \begin{multline*}\sum\nolimits_{i\in
    R_{j,j'}^c}{E\left[Y(\pi_j(i),j)Y(\pi_{j'}(i),j')\right]}/n\\=
\sum\nolimits_{i\in
  R_{j,j'}^c}{E\left[Y(\pi_j(i),j)\right]E\left[Y(\pi_{j'}(i),j')\right]}/n,
   \end{multline*}
where $R_{j,j'}^c=[n]\setminus R_{j,j'}$ and the expectations have been taken independently, because the
 rows in $R_{j,j}^c$ are permuted independently at random. The result then follows from
  the observation that $E\left[Y(\pi_l(i),l)\right]=a_l(i)$ for any $i\in R_{j,j'}^c$.\footnote{We have also verified
experimentally that the analytically
derived covariance matrix matches the covariance matrix estimated from a
sample from the distribution.}  \qed
\end{proof}

Now, the projection in
which the distributions differ most is given by
\begin{equation}
\label{eq:gain}
u_{\mathcal{H}}=\argmax_{u \in \mathbb{R}^m} G(u,\mathcal{H}).
\end{equation}
The vector $u_{\mathcal{H}}$ gives the direction in which the two distributions
differ the most in terms of the variance. Here we could in principle use
some other difference measure as well. We chose the form of Eq.~\eqref{eq:gain}
because it is intuitive and  it can be implemented efficiently, as described in the
following theorem.
\begin{theorem}\label{thm:opt}
The solution to the optimisation problem of Eq.~\eqref{eq:gain} is
given by $u_{\mathcal{H}}=Wv$, where $v$ is the first principal component of
$W^T\Sigma_1W$ and $W\in{\mathbb{R}}^{m\times m}$ is a whitening matrix \cite{doi:10.1080/00031305.2016.1277159}  such that
$W^T\Sigma_2W=I$.
\end{theorem}
\begin{proof}Using $u=Wv$ we can rewrite
  the gain of Eq.~\eqref{eq:g} as
  \begin{align}\label{eq:Wv}
    G(Wv,\mathcal{H})&=v^TW^T\Sigma_1Wv/v^TW^T\Sigma_2Wv
     \\
    &=v^TW^T\Sigma_1Wv/v^Tv. \nonumber
  \end{align}
Eq.~\eqref{eq:Wv} is maximised when $v$ is the maximal variance
direction of $W^T\Sigma_1W$, from which it follows that the solution
to the optimisation problem of Eq. \eqref{eq:gain} is given by
$u_{\mathcal{H}}=Wv$, where $v$ is the first principal component of
$W^T\Sigma_1W$. \qed
\end{proof}
In visualisations
(when making two-dimensional scatterplots), we project the data to the
first two principal components, instead of considering only the first component.

We note that at the limit of no background knowledge and with the most
general hypotheses, our method reduces to the PCA of the correlation
matrix, as shown by the following theorem.
\begin{theorem}\label{thm:pca}
In the special case of the first step in unguided data exploration,
i.e., comparing a pair of hypotheses specified by
$\mathcal{H} =\langle \emptyset+\mathcal{T}_{H_1}, \emptyset+
\mathcal{T}_{H_2}\rangle$, where $\mathcal{T}_{H_1}=\{([n],[m])\}$
and $\mathcal{T}_{H_2}=\cup_{j=1}^m\{([n],\{j\})\}$, the solution to
Eq.~\eqref{eq:gain} is given by the first principal component of
the correlation matrix of the data when the data has been scaled to unit variance.
\end{theorem}
\begin{proof}
The proof follows from the observations that for
$\mathcal{T}_{H_2}$ the covariance matrix $\Sigma_2$ is a diagonal
matrix (here a unit matrix),  resulting in the whitening matrix
$W=I$. For this pair of hypotheses, $\Sigma_1$ denotes the covariance matrix of the
original data. The result then follows from Thm.~\ref{thm:opt}. \qed
\end{proof}

\subsection{Selecting attributes for a tile constraint}
\label{sec:attributeselection}
Once we have defined the most informative projection, which displays
the most prominent differences between the distributions parametrised
by the pair of hypotheses, we can view the data in this
projection. This allows the user to observe different patterns, e.g.,
a clustered set of points, a linear relationship or a set of outlier
points.

After observing a pattern, the user defines a tile $(R,C)$ to be added
to $\mathcal{T}_u$.  The set of data points $R$ involved in the
pattern can be easily selected from the projection shown.  For
selecting the attributes that characterise the pattern, we can use a
procedure where for each attribute the ratio between the
standard deviation of the attribute for the selection and the standard
deviation of all data points is computed. If this ratio is below a
threshold value $\tau$ (e.g., $\tau=0.5$), then the attribute is
included in the set of attributes $C$ characterising the pattern. The
intuition here is that we are looking for attributes in which the
selection of points are more similar to each other than is expected
based on the whole data.

\section{Experiments}
\label{sec:experiments}
In this section we first consider the stability and scalability of the
method presented in this paper. After this, we present two brief
examples of how the method is used to (i) explore relations in a
dataset and (ii) focus on investigating a hypothesis concerning
relations in a subset of the data.

\textbf{Dataset} In the experiments, we utilise the \textsc{german}
socio-economic dataset \cite{boley2013, kang2016}\footnote{
  \url{http://users.ugent.be/~bkang/software/sica/sica.zip}}. The
dataset contains records from 412 administrative districts in
Germany. The full dataset has 46 attributes describing
socio-economic, political and geographic aspects of the districts, but
we only use 32 variables (see the Sections \ref{sec:german1} and \ref{sec:german2} in the Appendix for details) in
the experiments. We scale the real-valued variables to zero mean
and unit variance. All of the experiments were performed with a
single-threaded R 3.5.0 \cite{Rproject} implementation on a MacBook
Pro laptop with a \unit[2.5]{GHz} Intel Core i7 processor.

\subsection{Stability and scalability}

We first study the sensitivity of the results with respect to noise or
missing data rows. We begin the experiment by separating 32
real-valued variables and 3 (non-trivial) factors from the full {\sc
  german} data. A synthetic dataset, parametrised by the noise term
$\sigma$ and an integer $\Delta n$ is constructed as follows. First,
we randomly remove $\Delta n$ rows, after which we to the remaining
variables add Gaussian noise with variance $\sigma^2$, and finally
rescale all variables to zero mean and unit variance. We create a
random tile by randomly picking a factor that defines the rows in a
tile and randomly pick 2--32 columns. The background distribution
consists of three such random tiles and the hypothesis tiles are
constructed of one such random tile $(R,C)$ as
$\mathcal{T}_{H_1}=\{(R,C)\}$ and $\mathcal{T}_{H_2}=\cup_{j\in
  C}{\{(R,\{j\})\}}$.

\begin{table}[t!]\centering
\centering
\begin{footnotesize}
\begin{tabular}[t]{ccc}
\toprule
  $\sigma$ & $\Delta n$ &  error \\
\midrule
$0$ & $0$ & $0.000$ \\
$0$ & $200$ & $0.038$ \\
$1$ & $0$ & $0.088$ \\
$1$ & $200$ & $0.135$ \\
$10$ & $0$ & $0.562$ \\
$10$ & $200$ & $0.664$ \\
\bottomrule
\end{tabular}
\end{footnotesize}
\caption{The mean error as function of perturbance to data (noise
  with variance $\sigma$ added and $\Delta n$ random rows removed).
  The error is here the difference in gain of Eq. \eqref{eq:g}
between the optimal solution $u_{\mathcal{H}}$ to the
  solution $u_{\mathcal{H}}^*$ found on perturbed data.}
\label{tab:r:a}
\end{table}

\begin{table}[t!]\centering
\centering
\begin{footnotesize}
\begin{tabular}[t]{cccc}
\toprule
  $n$ & $m$  & $t_\mathrm{model}$ (s) & $t_\mathrm{view}$ (s) \\
\midrule
$10^3$ & $10$ & $0.03$ & $0.01$ \\
$10^4$ & $10$ & $0.46$ & $0.08$ \\
$10^3$ & $100$ & $0.16$ & $0.66$ \\
$10^4$ & $100$ & $1.78$ & $5.52$ \\
$10^3$ & $200$ & $0.19$ & $2.58$ \\
$10^4$ & $200$ & $4.19$ & $22.35$ \\
\bottomrule
\end{tabular}
\end{footnotesize}
\caption{Median
  wall clock running time for random data with varying number of rows
  ($n$) and columns ($m$) for a dataset consisting of Gaussian random
  numbers. We give the time to add three  random tiles plus
  hypothesis tiles ($t_\mathrm{model}$) and the time to find the most
  informative view ($t_\mathrm{view}$), i.e., to solve Eq.~\eqref{eq:gain}.
}
\label{tab:r:b}
\end{table}

The results are shown in Tab.~\ref{tab:r:a}. We notice that the method
is relatively insensitive with respect to the gain to noise and
removal of rows. Even removing about half of the rows does not change
the results meaningfully. Only very large noise, corresponding to
$\sigma=10$ (i.e., c. 10\% signal to noise ratio) degrades the results
substantially.

\begin{figure*}[htbp]
\centering
\begin{tabular}{cc}
\includegraphics[width=0.4\textwidth]{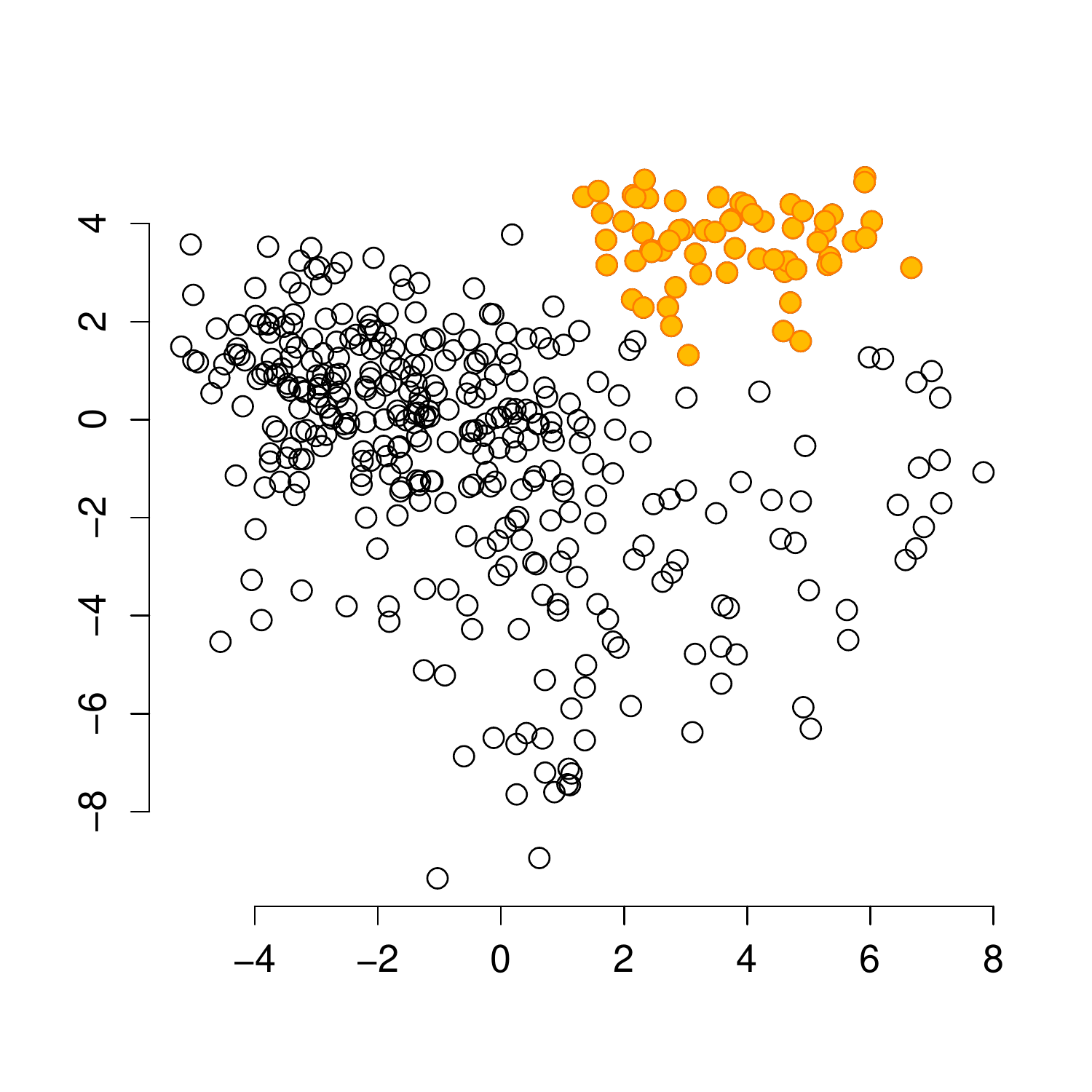} &
\includegraphics[width=0.4\textwidth]{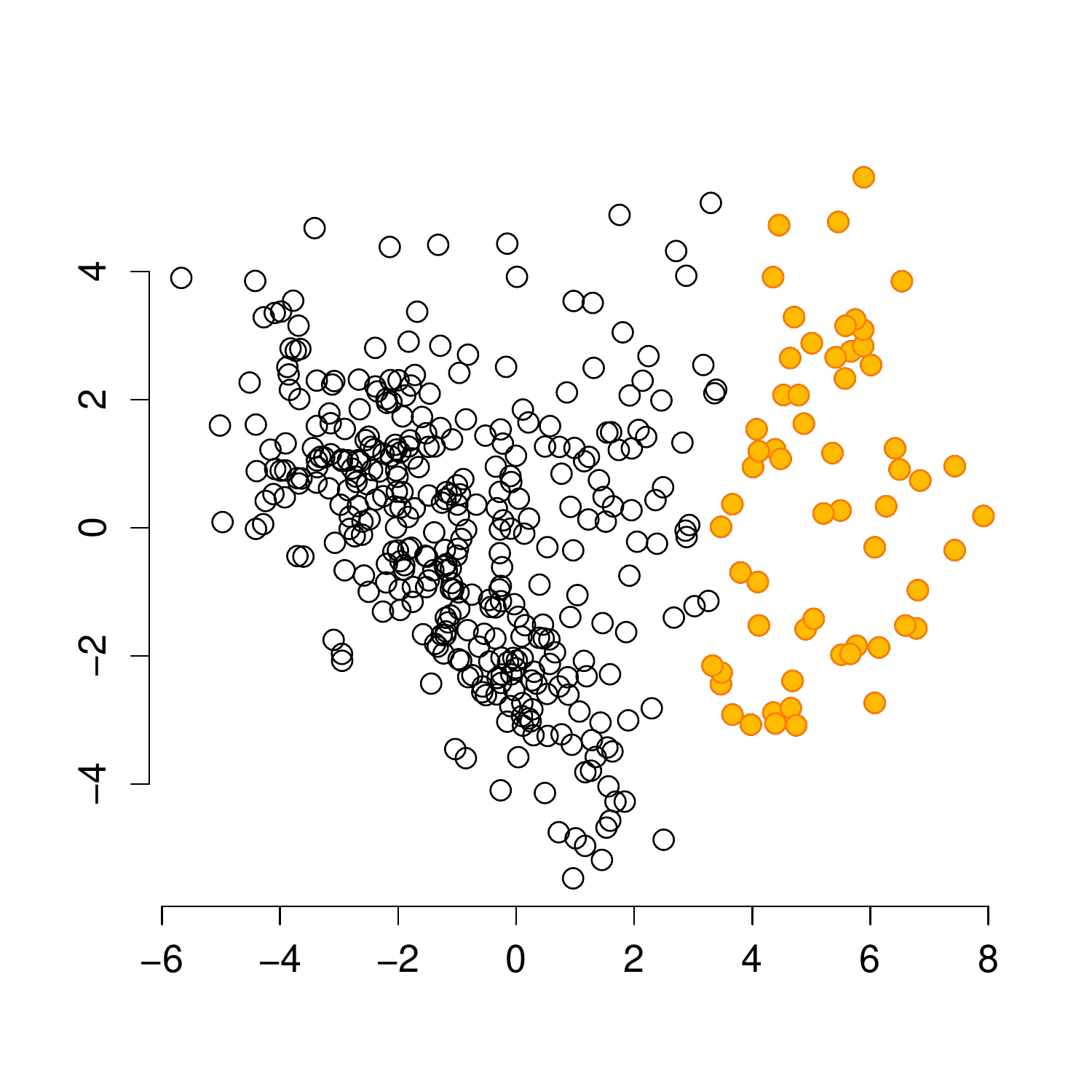}
\end{tabular}
  \caption{Views of the \textsc{german} dataset  corresponding to 
the hypothesis pairs $\mathcal{H}_{E,\emptyset}$ (left) and $\mathcal{H}_{E,\{t\}}$ (right).
Black circles ($\circ$) show  data points; selected points are marked with orange.}
  \label{fig:exploration1}
 \end{figure*}

Tab. \ref{tab:r:b} shows the running time of the algorithm as a
function of the size of the data for Gaussian random data with a
similar tiling setup as used for the {\sc german} data. We make two
observations. First, the tile operations scale linearly with the size
of the data $nm$ and they are relatively fast. Most of the time is spent on
finding the views, i.e., solving Eq.~\eqref{eq:gain}. Even our
unoptimised pure R implementation runs in seconds for datasets that
are visualisable (having thousands of rows and hundreds of
attributes); any larger dataset should in any case be downsampled for
visualisation purposes.

\subsection{Human-guided data exploration of the \textsc{german} dataset}
Finally, we demonstrate our human-guided data exploration framework by
exploring the \textsc{german} dataset under different
hypotheses. Sections \ref{sec:german1} and \ref{sec:german2} in the
Appendix contain larger figures with more details (samples
corresponding to both hypotheses and axes labelled with components of
the projection vectors) and more thorough explanations of the
exploration process described below.

We start with {\em unguided data exploration}
where we have no prior knowledge about the
data and our interest is as generic as possible.
In this case $\mathcal{T}_u=\emptyset$ and as the hypothesis
tilings we use $\mathcal{T}_{E_1}$, where all rows and columns belong
to the same tile (a fully-constrained tiling), and
$\mathcal{T}_{E_2}$, where all columns form a tile of their own (fully
unconstrained tiling). Our hypothesis pair is then $\mathcal{H}_{E,
  \emptyset}=\langle\emptyset+ \mathcal{T}_{E_1},
\emptyset+\mathcal{T}_{E_2}\rangle$.

\begin{table}[t!]
\centering
\begin{tabular}{l cccc}
\toprule
& $\mathcal{H}_{E,\emptyset}$ & $\mathcal{H}_{E, \{t\}}$ &
$\mathcal{H}_{F, \emptyset}$ & $\mathcal{H}_{F, \{t\}}$ \\
\midrule
$u_{E, \emptyset}$& \textbf{8.831} & 3.887 & 1.921 &  1.124 \\
$u_{E, \{t\}}$        & 7.933 & \textbf{8.920} & 1.172 & 1.100 \\
$u_{F, \emptyset}$& 4.879 &  2.062 & \textbf{2.958} &  1.087 \\
$u_{F, \{t\}}$ & 1.618 & 1.842 & 1.489 &  \textbf{1.773} \\
$u_\mathrm{pca}$ & \textbf{8.831}& 3.887 & 1.921 &  1.124 \\
$u_\mathrm{ica}$  & 0.004 & 0.004  & 1.000 &  0.999 \\
 \bottomrule
\end{tabular}
\caption{The value of the gain  $G(u,\mathcal{H})$  for different projection vectors $u$ and hypothesis pairs $\mathcal{H}$.}
\label{tab:gains}
\end{table}

Now, looking at the view where the distributions characterised by the
pair $\mathcal{H}_{E, \emptyset}$ differ the most, shown in
Fig.~\ref{fig:exploration1} (left), we observe cluster patterns. Selecting a
set of data points allows us to investigate what kind of items and
attributes the selected points represent. E.g., Cluster 1 (shown in
orange in Fig.~\ref{fig:exploration1} (left)) corresponds to rural districts in
Eastern Germany characterised by a high degree of voting for the Left
party. We now add a tile constraint $t$ for the items in the observed
pattern where the columns (attributes) are chosen as described in
Sec.~\ref{sec:attributeselection}, using a threshold value $\tau =
2/3$.  The hypothesis pair is then updated to $\mathcal{H}_{E,
  \{t\}}=\langle \{t\}+ \mathcal{T}_{E_1}, \{t\}+\mathcal{T}_{E_2}
\rangle$. The most informative view displaying differences of the
distributions parametrised by $\mathcal{H}_{E, \{t\}}$ is now shown in
Fig.~\ref{fig:exploration1} (right)  and we observe that Cluster 2 (the selection
shown in orange) has become prominent. By inspecting the class
attributes of this selection we learn that these items correspond to
urban districts.

 \begin{figure*}[htbp]
 \centering
\begin{tabular}{cc}
\includegraphics[width=0.4\textwidth]{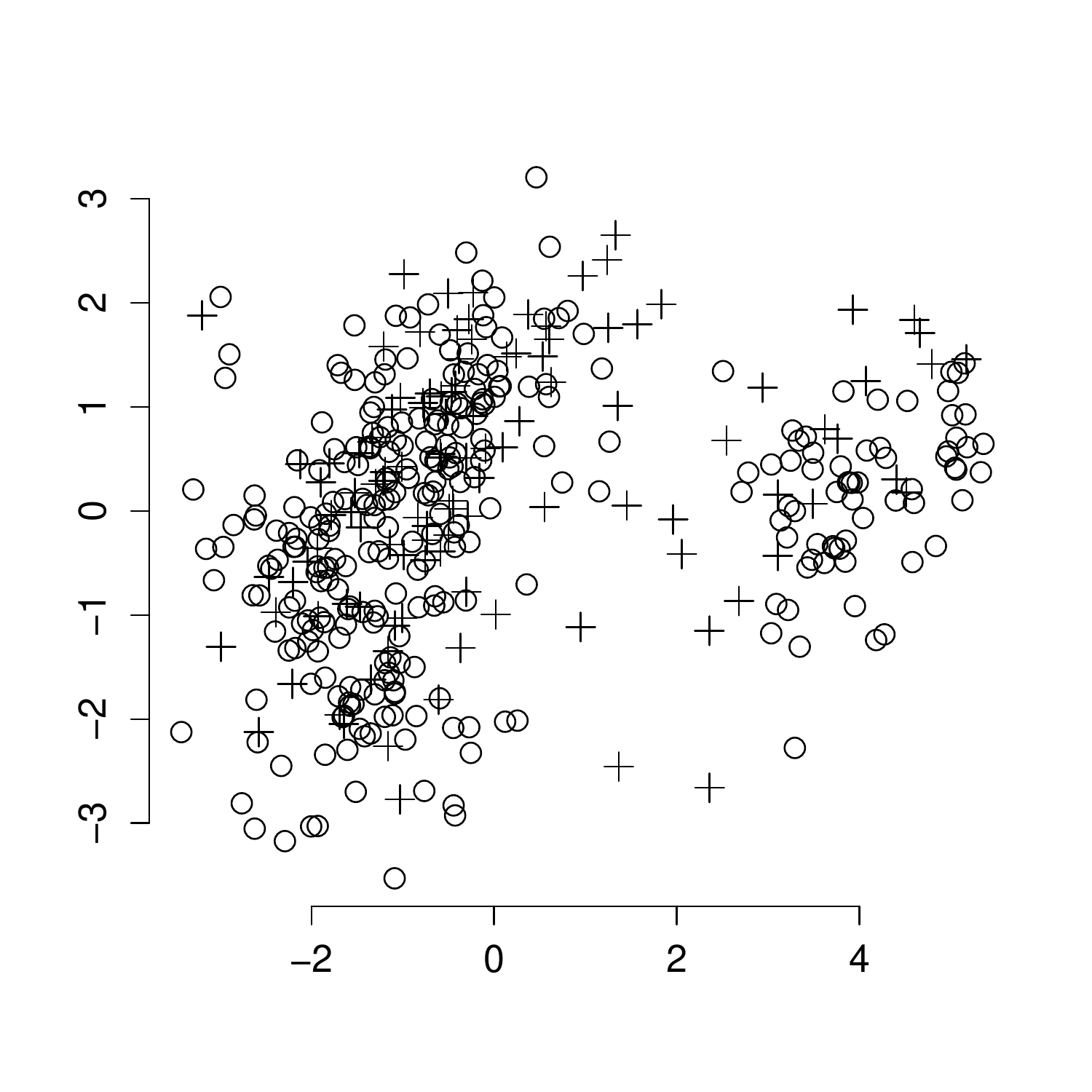} &
\includegraphics[width=0.4\textwidth]{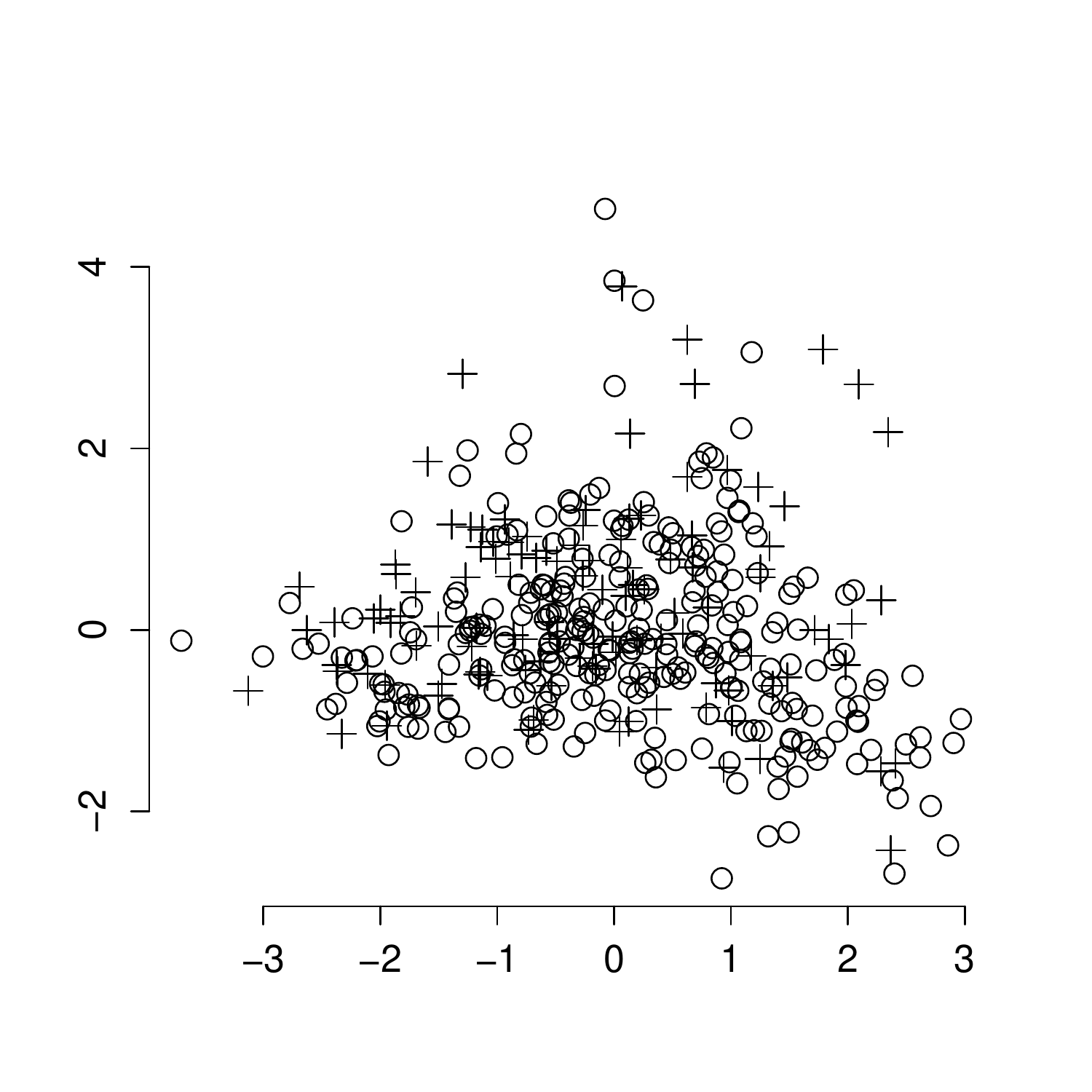}
\end{tabular}
  \caption{Views of the \textsc{german} dataset  corresponding to 
the hypothesis pairs $\mathcal{H}_{F,\emptyset}$ (left) and $\mathcal{H}_{F,\{t\}}$ (right). The
circles~$\circ$ show data points inside the focus area; the points outside the focus
 area are shown using $+$.}
  \label{fig:exploration2}
 \end{figure*}

Next, we focus on a more specific hypothesis involving only a subset
of rows and attributes. As the subset of rows $R$ we choose rural
districts. As attributes of interest, we consider a partition $C =
C_1\cup C_2\cup C_3\cup C_4$, where $C_1$ contains voting results for
the political parties in 2009, $C_2$ contains demographic attributes,
$C_3$ contains workforce attributes, and $C_4$ contains education,
employment and income attributes. We here want to investigate
relations between different attribute groups, ignoring the relations
inside the groups.

We form the hypothesis pair $\mathcal{H}_{F, \emptyset}=\langle
\emptyset+\mathcal{T}_{F_1}, \emptyset+\mathcal{T}_{F_2} \rangle$,
where $\mathcal{T}_{F_1}$ consists of a tile spanning all rows in $R$
and all columns in $C$ whereas $\mathcal{T}_{F_2}$ consists of four
tiles: $t_i = (R, C_i)$, $i\in\{1,2,3,4\}$. Looking at the view where the
distributions parametrised by the pair $\mathcal{H}_{F, \emptyset}$
differ the most, shown in Fig.~\ref{fig:exploration2} (left), we find two clear
clusters corresponding to a division of the districts into those
located in the East, and those located elsewhere. We could also have
used our already observed background knowledge of Cluster 1, by
considering the hypothesis pair $\mathcal{H}_{F, \{t\}}=\langle
\{t\}+\mathcal{T}_{F_1}, \{t\}+\mathcal{T}_{F_2} \rangle$.  For this
hypothesis pair, the most informative view is shown in
Fig.~\ref{fig:exploration2} (right), which clearly is different to
Fig.~\ref{fig:exploration2} (left), demonstrating that the background knowledge
is important.

To understand the utility of the views shown, we compute values of the
gain function as follows. We consider our four hypothesis pairs
$\mathcal{H}_{E, \emptyset}$, $\mathcal{H}_{E, \{t\}}$,
$\mathcal{H}_{F, \emptyset}$, and $\mathcal{H}_{F, \{t\}}$. For each
of these pairs, we denote the direction in which the two distributions
differ most in terms of the variance (solutions to Eq.~\eqref{eq:gain})
by $u_{E, \emptyset}$, $u_{E, \{t\}}$, $u_{F, \emptyset}$, and $u_{F,
  \{t\}}$, respectively.  We then compute the gain $G(u,\mathcal{H})$
for each $u\in \{u_{E, \emptyset}, u_{E, \{t\}}, u_{F, \emptyset},
u_{F, \{ t\}} \}$ and $\mathcal{H}\in\{\mathcal{H}_{E, \emptyset},
\mathcal{H}_{E, \{t\}}, \mathcal{H}_{F, \emptyset}, \mathcal{H}_{F,
  \{t\}}\}$. For comparison, we also compute the first PCA and ICA
projection vectors, denoted by $u_\mathrm{pca}$ and $u_\mathrm{ica}$,
respectively, and calculate the gain in different hypothesis pairs
using these. The results are presented in Tab.~\ref{tab:gains}. We
notice that the gain is indeed always the highest when the projection
vector matches the hypothesis pair (highlighted in the table), as
expected. This shows that the views presented are indeed the most
informative ones given the current background knowledge and the
hypotheses being investigated. We also notice that the gain for PCA
is equal to that of unguided data exploration, as expected by
Thm.~\ref{thm:pca}.
When some
background knowledge is used or if we investigate a particular
hypothesis, the views achievable using PCA or ICA are less informative
than using our framework. The gains close to zero for ICA
are directions where the variance of the more constrained distribution
is small due to, e.g., linear dependencies in the data.

\section{Related work}\label{sec:related}
Iterative data mining \cite{hanhijarvi:2009} is a paradigm where
patterns already discovered by the user are taken into account as
constraints during subsequent exploration. This concept of iterative
pattern discovery is also central to the data mining framework
presented in \cite{debie:2011a, debie:2011b, debie2013}, where the
user's current knowledge (or beliefs) of the data is modelled as a
probability distribution over datasets, and this distribution is then
updated iteratively during the exploration phase as the user discovers
new patterns. Our work has been motivated by
\cite{puolamaki2010,puolamaki:2016, kang:2016b,puolamaki2017},
where these concepts have been successfully applied
in visual exploratory data analysis where
the user is visually shown a view of the data
which is maximally informative given the user's current
knowledge.
Visual interactive exploration has also been applied in
different contexts, e.g., in item-set mining and subgroup
discovery \cite{boley2013, dzyuba2013, vanleeuwen2015, paurat2014},
information retrieval \cite{ruotsalo2015}, and network
analysis~\cite{chau2011}.

Solving the problem of determining which views of the data that are
maximally informative to the user, and hence interesting, have been
approached in terms of, e.g., different projections and measures of
interestingness \cite{debie:2016:a, kang2016,
vartak:2015:a}. Constraints have also been used to assess the
significance of data mining results, e.g., in pattern
mining \cite{lijffijt2014} or in investigating spatio-temporal
relations \cite{chirigati:2016:a}.

The present work fundamentally differs from the above discussed
previous work concerning iterative data mining and applications to
visual exploratory data analysis in the following way. In previous
work, the user is presented with informative views (visual or not) of
the data, but \emph{the user cannot beforehand know which aspects of
  the data these views will show}, since by definition the views are
such that they contrast maximally with the user's current
knowledge. The implication is that the user cannot steer the
exploration process. In the present work we solve this navigational
problem by incorporating both the user's knowledge of the data, and
different hypotheses concerning the data into the background
distribution.

\section{Conclusions}\label{sec:conclusions}
In this paper we proposed a method to integrate both the user's
background model learned from the data and the user's current
interests in the explorative data analysis process. We provided an
efficient implementation of this method using constrained
randomisation. Furthermore, we extended PCA to work seamlessly with
the framework in the case of real-valued datasets.

The power of human-guided data exploration stems from the fact that
typical datasets contain a huge number of interesting
patterns. However, the patterns that are interesting to a user depend
on the task at hand. A non-interactive data mining method is therefore
restricted to either show generic features of the data---which may
already be obvious to an expert---or then output unusably many
patterns (a typical problem, e.g., in frequent pattern mining: there
are easily too many patterns for the user to absorb). Our framework is
a solution to this problem: by integrating the human background
knowledge and focus---formulated as a mathematically defined
hypothesis---we can at the same time guide the search towards topics
interesting to the user at any particular moment while taking the
user's prior knowledge into account in an understandable and efficient
way.

This work could be extended, e.g., to understand classifier or
regression functions in addition to static data and to different data
types, such as time series. An interesting problem would also be to
find an efficient algorithm that could find a sparse solution to the
optimisation problem of Eq.~\eqref{eq:gain}. To our knowledge, no such
solution is readily available as the solutions for sparse PCA are not
directly applicable here; sparse PCA would give a sparse variant of the
vector $v$ in Thm.~\ref{thm:opt}, which would however not result in a
sparse $u_{\mathcal{H}}=Wv$. An obvious next step would be to
implement these ideas in interactive data analysis tools.

%% file: appendix.tex
\section{Algorithm for merging tiles}
\label{sec:merge}

Merging a new tile into a tiling where all tiles are non-overlapping
can be done efficiently using Alg.~\ref{alg:merge}. We assume that the
starting point always is a non-overlapping set of tiles and hence we
only need to consider the overlap that the new tile has with the tiles
already existing in the tiling. This is similar to the merging of
statements considered in \cite{kalofolias:2016}.

The algorithm for merging tiles has two steps. Let $\mathcal{T}$ be
the existing tiling (with non-overlapping tiles) and let $t=(R,C)$ be
the new tile to be added to the tiling, where $R$ is the set of rows
and $C$ is the set of columns spanned by the tile $t$. In the first
step (lines 1--11) we identify those tiles in the tiling $\mathcal{T}$
with which $t$ overlaps. In the second step we resolve (merge) the
overlap between $t$ and the tiles identified in $\mathcal{T}$. The
algorithm proceeds as follows.

An empty hash map is initialised (line 1) in order to be used to
detect overlap between columns of existing tiles in $\mathcal{T}$ and
the new tile $t$. We proceed to iterate over each row $R$ in the new
tile (lines 2--11).

The matrix $\mathcal{T}$ is of the same size as the data matrix and
contains information on which tiles that cover a particular part of
the data matrix. Each element in $\mathcal{T}$ hence holds information
concerning which tile that occupies a particular region. Since the
tiling described by $\mathcal{T}$ is non-overlapping, it means that
each element in $\mathcal{T}$ corresponds to the ID of the tile that
covers that position. Given a row $i \in R$ and a set of columns $C$
(line 3) we then get the IDs of the tiles on row $i$ with which $t$
overlaps and we store this in $K$. The hash map is used to detect if
this row has been seen before, i.e., if $K$ is a key in $S$ (line
4). If this is the first time this row is seen, $K$ is used as the key
for a new element in the hash map and $S(K)$ is initialised to be a
tuple (line 5). Elements in this tuple are referred to by name, e.g.,
$S(K)_\mathrm{rows}$ gives the set of rows associated with the key
$K$, while $S(K)_\mathrm{id}$ gives the set of tile IDs.

On lines 6 and 7 we store the current row index $S(K)_\mathrm{rows}$
and the unique tile IDs $S(K)_\mathrm{id}$ in the tuple. If the row
was seen before, the row set associated with these tile IDs is updated
(line 9). After this first step, the hash map $S$ contains tuples of
the form (\emph{rows}, \emph{id}) where \emph{id} specifies the IDs of
the tiles with which $t$ overlaps at the rows specified by
\emph{rows}.

We then proceed to the second step in the algorithm (lines 12--16)
where the identified overlap is resolved. We first determine the
currently largest tile ID in use (line 12). After this we iterate over
the tuples in the hash map $S$. For each tuple we must update the
tiles having IDs $S(K)_\mathrm{id}$ and on line 14 we hence find the
columns associated with these tiles. After this, the IDs of the
affected overlapping tiles are updated (line 15), and the tile ID
counter is incremented. Finally, the updated tiling is returned on
line 17. The time complexity of the tile merging algorithm is
$\mathcal{O}{(n m)}$.

\begin{algorithm2e}[h]
 \caption{\label{alg:merge} Merging a new tile $t$ with tiles in a
   tiling $\mathcal{T}$. The function \texttt{HashMap} denotes a hash
   map. The value in a hash map $H$ associated with a key $x$ is
   $H(x)$ and $\texttt{keys}(H)$ gives the keys of $H$. The function
   \texttt{Tuple} creates a (named) tuple. An element $a$ in a tuple
   $w = (\mathrm{a}, \mathrm{b})$ is accessed as $w_\mathrm{a}$. The
   function \texttt{unique} returns the unique elements of an array.}
 \begin{small}
 \SetKwInOut{Input}{input}
 \SetKwInOut{Output}{output}

 \Input{(1) A tiling $\mathcal{T}$ ($n \times m$
   data matrix where each element is the ID of the tile
   it belongs to) \\(2) A new tile $t = (R,
   C)$.}
 \Output{$\mathcal{T} + \{t\}$ (the tiling $\mathcal{T}$ merged with $t$).}

 $S \leftarrow \texttt{HashMap}$\;

 \For{$i \in R$} {
   $K \leftarrow \mathcal{T}(i, C)$\;
   \uIf { $K \notin \texttt{keys}(S) $}{
     $S(K) \leftarrow \texttt{Tuple}$\;
     $S(K)_\mathrm{rows} \leftarrow \{i\}$\;
     $S(K)_\textrm{id} \leftarrow \texttt{unique}( \mathcal{T}(i, C))$\;
   }      \Else{
     $S(K)_\textrm{rows} \leftarrow  S(K)_\textrm{rows} \cup \{i\}$\;
   }
 }

 $p_\mathrm{max} \leftarrow \max(\mathcal{T}(R,C))$\;
 \For{$K \in \texttt{keys}(S)$}{
   $C' = \left\{ c \mid \mathcal{T}(S(K)_\mathrm{rows}, c) \in S(K)_\textrm{id}  \right\}$ \;
   $\mathcal{T}\left(S(K)_\textrm{rows}, C'\right) \leftarrow p_\mathrm{max} + 1$; \qquad  $p_\mathrm{max} \leftarrow p_\mathrm{max} + 1$\;
     }
\BlankLine
 \Return{$\mathcal{T}$}
 \BlankLine
 \end{small}
\end{algorithm2e}

\begin{figure}[h]
  \centering
  \centering
  \includegraphics[width=0.8\columnwidth]{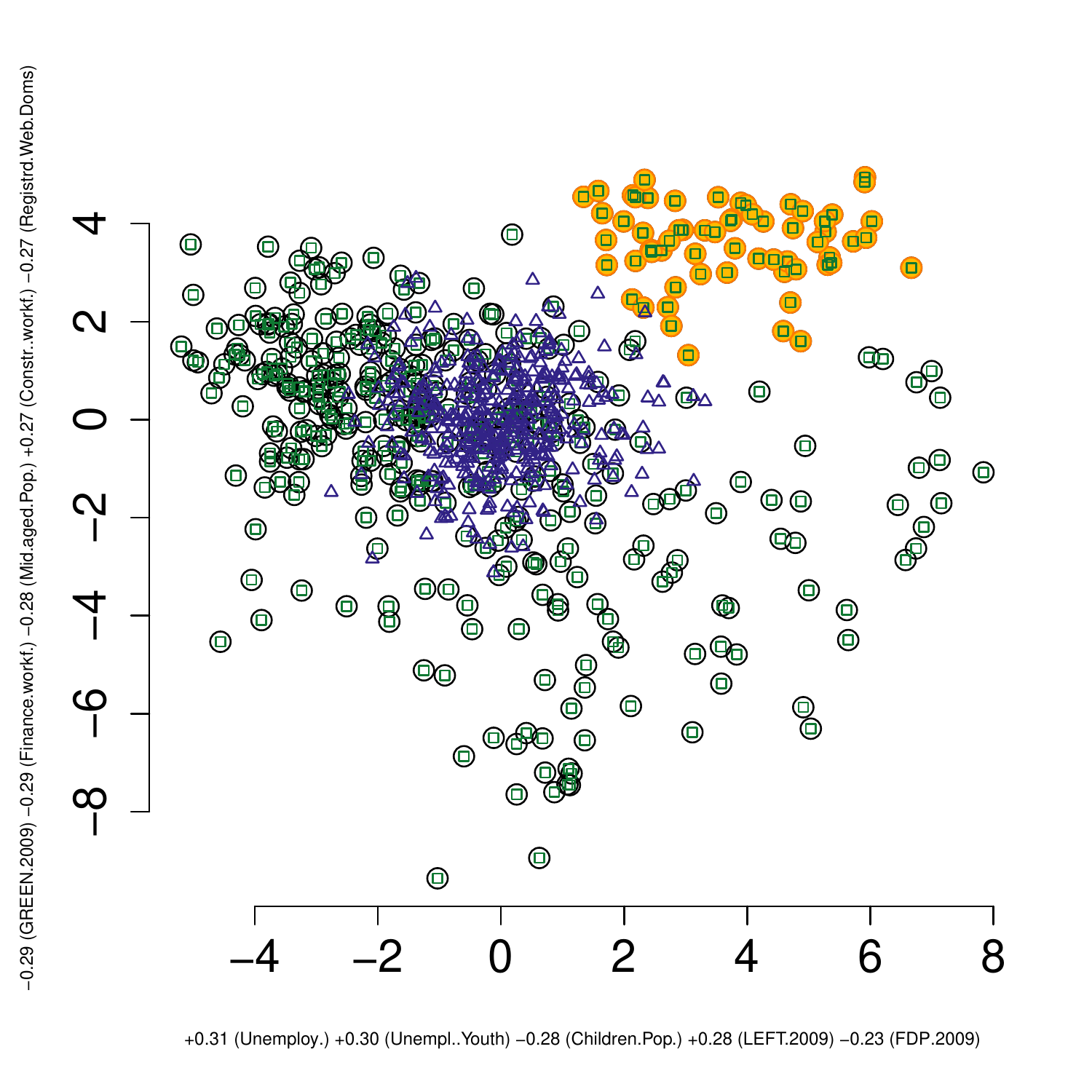}
  \caption{Step 1 in the exploration of the dataset. The data points are shown by
    black spheres. Samples from the first part of the hypothesis pair are shown with
    green squares and samples from the second part with blue
    triangles. The selected cluster of points (see the text for
    discussion) is shown with orange colour. The axis labels show
    the five coefficients with the largest absolute values of each
    projection vector.}
  \label{fig:explore:vee}
\end{figure}
\section{Exploration of the German data without background knowledge}
\label{sec:german1}

\subsection*{Dataset}
In the experiments we consider the \textsc{german} socio-economic
dataset \cite{boley2013, kang2016}.
  The
dataset contains records from 412 administrative districts in
Germany. Each district is represented by 46 attributes describing
socio-economic and political aspects in addition to attributes such as
the type of the district (rural/urban), area name/code, state, region
and the geographic coordinates of each district center. The
socio-ecologic attributes include, e.g., population density, age and
education structure, economic indicators (e.g., GDP growth,
unemployment, income), and the proportion of the workforce in
different sectors. The political attributes include election results
of the five major political parties (CDU/CSU, SPD, FDP, Green, and
Left) in the German federal elections in 2005 and 2009, as well as the
voter turnout. For our experiments we remove the election results from
2005, all non-numeric variables, and the area code and coordinates of
the districts, resulting in 32 real-valued attributes (although we use
the full dataset when interpreting the results). We scale the
real-valued variables to zero mean and unit variance.

\subsection*{Step 1}
We first consider the case where the data miner has no prior knowledge
concerning the relations in the data, i.e., we have no initial
background knowledge. We hence set $\mathcal{T}_u = \emptyset$ and as
the hypothesis tilings we use $\mathcal{T}_{E_1}$, where all rows and
columns belong to the same tile (a fully-constrained tiling), and
$\mathcal{T}_{E_2}$, where all columns form their own tile (fully
unconstrained tiling). Our hypothesis pair is then $\mathcal{H}_{E,
  \emptyset}=\langle \emptyset+\mathcal{T}_{E_1},
\emptyset+\mathcal{T}_{E_2}\rangle$.

We then consider the first view of the data
(Figure~\ref{fig:explore:vee}), which is maximally informative in terms
of the two distributions parametrised by the hypotheses differing the
most. In our figures we use black circles ($\circ$) to denote the real
data points. We mark selected clusters with orange. Green squares
($\square$)  denote points belonging to the data sample
from the distribution parametrised by
$\mathcal{T}_u+\mathcal{T}_{E_1}$ and blue triangles ($\triangle$)
denote points belonging to the data sample from the distribution
parametrised by $\mathcal{T}_u+\mathcal{T}_{E_2}$. Note that in this
particular case the actual data and the sample corresponding to
$\mathcal{T}_u+\mathcal{T}_{E_1}$ are identical (although the rows
might be in different order), since this is a fully constrained
tiling. The $x$ and $y$ axis labels show the five attributes with the
largest absolute values of each projection vector.

\begin{figure*}[h]
  \centering
  \begin{subfigure}[b]{0.4\textwidth}
  \includegraphics[width=0.8\columnwidth]{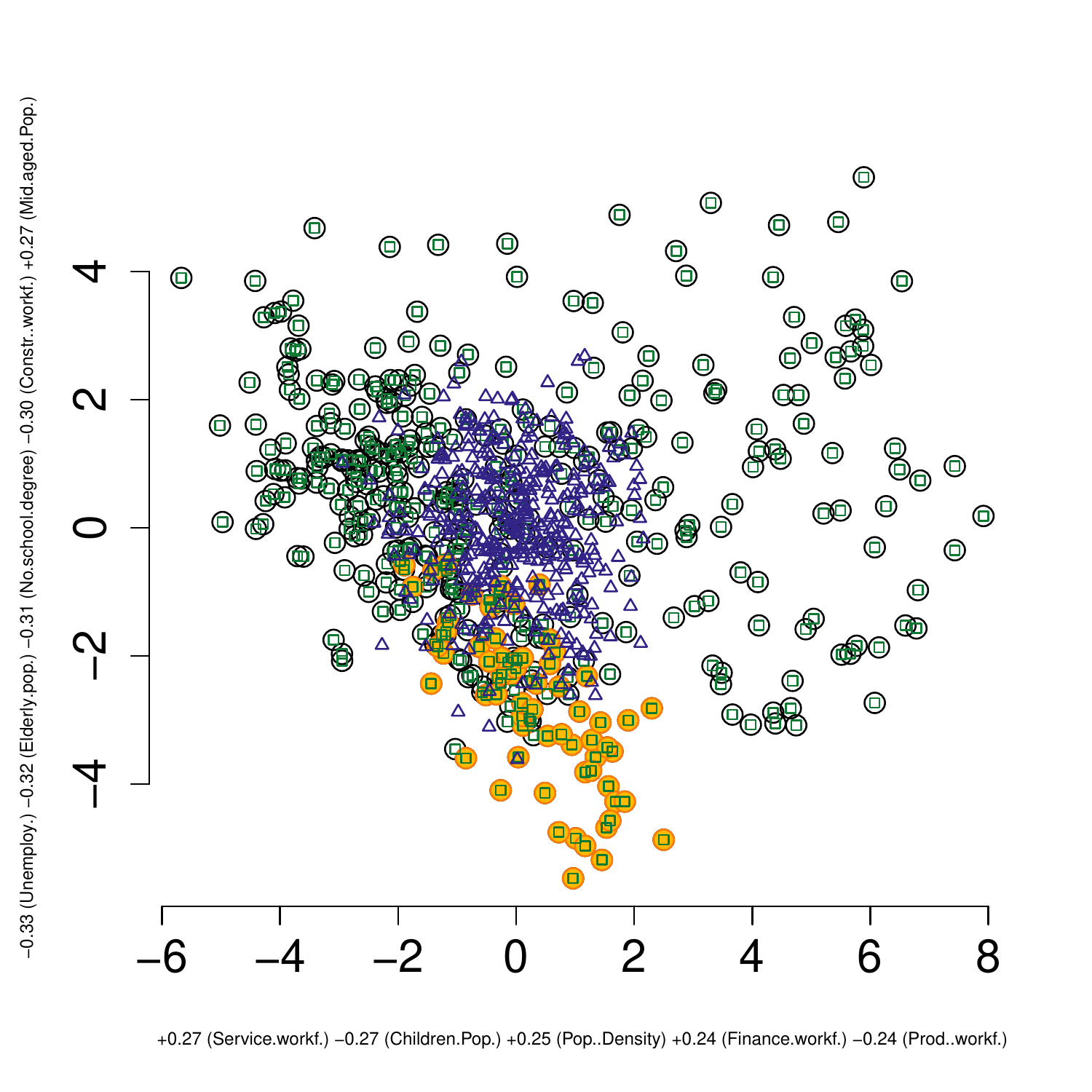}
  \caption{Hypotheses updated with the knowledge concerning Cluster 1
    (marked here for illustration purposes).}
 \label{fig:explore:veeb}
  \end{subfigure}
 \begin{subfigure}[b]{0.4\textwidth}
  \includegraphics[width=0.8\columnwidth]{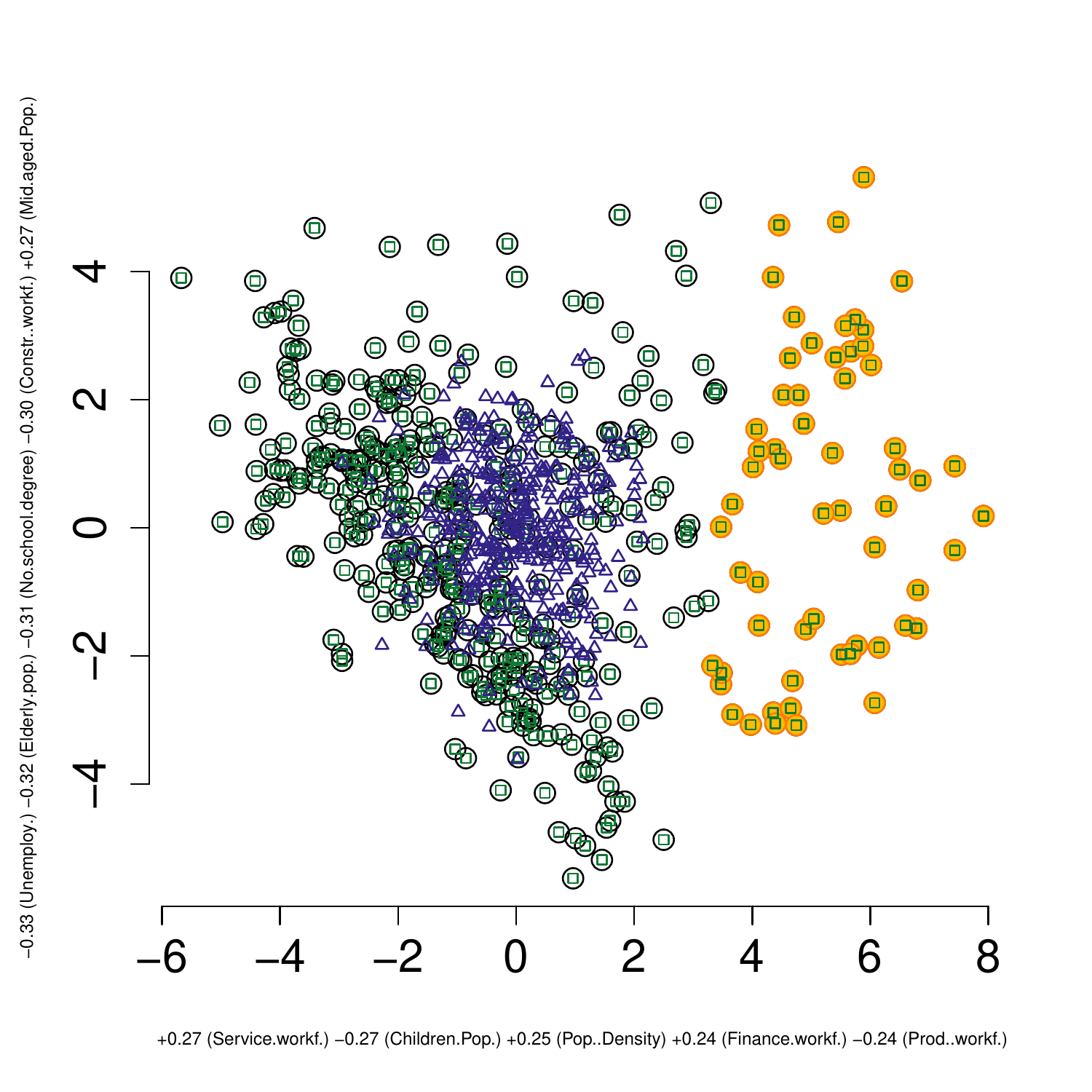}
  \caption{Cluster 2 marked.}
  \label{fig:explore:vet}
  \end{subfigure}
  \caption{Step 2 in the exploration of the dataset. The meaning of
    the points is as in Fig.~\ref{fig:explore:vee}. Left: Hypotheses
    updated with the knowledge of Cluster 1 (marked here for
    illustration purposes). Right: Cluster 2 marked.}
  \label{fig:explore:2}
\end{figure*}

We observe that there are cluster patterns visible in the data and
that the two distributions differ (the green and blue points are
distributed differently). In order to investigate the characteristics
of the data points, corresponding to different patterns in the
\textsc{german} data, we select a set of points that form a
cluster. Here we first choose to focus on the set of points in the
upper right corner, marked with orange in
Figure~\ref{fig:explore:vee}.

We now want to learn about the cluster we have identified.  We
consider the \emph{Type} and \emph{Region} attributes for the subset
of data points (the marked cluster in the view) in the original
data. These two categorical attributes tell whether a district in a
cluster is urban or rural (\emph{Type}) and where in Germany it is
located (\emph{Region}); in the Northern, Southern, Western, or
Eastern region. For Cluster 1, we obtain the information shown in
Table~\ref{tab:explore:c1}.

\begin{table}[h]
\centering
\begin{footnotesize}
\begin{subtable}{0.21\textwidth}
\begin{tabular}{ll}
  \toprule
  Region &    Type \\
  \midrule
East :62   & Rural:62   \\
  North: 0   & Urban: 0   \\
  South: 0   &  \\
  West : 0   &  \\
   \bottomrule
\end{tabular}
\caption{Cluster 1.}
  \label{tab:explore:c1}
\end{subtable}
\begin{subtable}{0.21\textwidth}
\begin{tabular}{ll}
  \toprule
  Region &    Type \\
  \midrule
East :22   & Rural: 0   \\
  North:10   & Urban:60   \\
  South: 7   &  \\
  West :21   &  \\
   \bottomrule
\end{tabular}
\caption{Cluster 2.}
  \label{tab:explore:c2}
\end{subtable}
\end{footnotesize}
\caption{{\em Region} and {\em Type} attributes for clusters observed in the data.}
  \label{tab:explore:clusters}
\end{table}

We also consider a parallel coordinates plot of the data, shown in
Figure~\ref{fig:explore:pcp:c1}. This plot shows all 32 attributes in the
data. The currently selected points (Cluster 1) are shown with red
while the rest of the data is shown in black. The number in 
parentheses following each variable name is the ratio of the standard
deviation of the selection and the standard deviation of all data. If
this number is small we can conclude that the values for a particular
attribute are homogeneous inside the selection (behave similarly).

Cluster 1 hence corresponds to Urban areas in the East. Based on the
parallel coordinates plot in Figure~\ref{fig:explore:pcp:c1}, one
clear political feature which we can observe is that there is little
support for the Green party and a high support for the Left party in
these districts.

\subsection*{Step 2}
We continue our exploration by adding a tile constraint for Cluster
1. The set of rows for the tile is determined by our selection (marked
in Figure~\ref{fig:explore:vee}). To determine the set of columns for
the tile constraint we use the following heuristic: using the parallel
coordinates plot (Figure~\ref{fig:explore:pcp:c1}) we choose as
columns for the tile those attributes for which the standard deviation
ratio (the number in parentheses) $\tau$ is less than $2/3$.

We then update our hypotheses to take into account the newly added tile, i.e.,
consider the  hypothesis pair $\mathcal{H}_{E,
  \{t\}}=\langle \{t\}+ \mathcal{T}_{E_1}, \{t\}+\mathcal{T}_{E_2}
\rangle$.
The most informative view is shown in Figure~\ref{fig:explore:veeb},
We have here for illustration purposes marked Cluster 1 with
orange. As can be seen, this cluster is no longer as clearly visible
as in our first view in Figure~\ref{fig:explore:vee}. This is
expected, since this pattern has been accounted for in the background
distribution and the relations in this cluster no longer differ
between the data samples corresponding to the distributions
characterised by the hypothesis pair we are currently
investigating. Instead we observe Cluster 2, marked in
Figure~\ref{fig:explore:vet}. In a similar fashion as we did for
Cluster 1, we consider the \emph{Region} and \emph{Type} attributes
for Cluster 2 (Table~\ref{tab:explore:c2}) and conclude that we have
found urban districts spread out over all regions. Based on the
parallel coordinates plot shown in Figure~\ref{fig:explore:pcp:c2} we
can conclude that these districts are characterised by a low fraction
of agricultural workforce and a high amount of service workforce, both
as expected in urban districts. We also notice that these districts
have had a high GDP growth in 2009 and that it appears that the amount
of votes for the CDU party in these districts was quite low.

\begin{figure*}[!h]
  \centering
  \begin{subfigure}[b]{0.45\textwidth}
  \includegraphics[width=0.8\columnwidth]{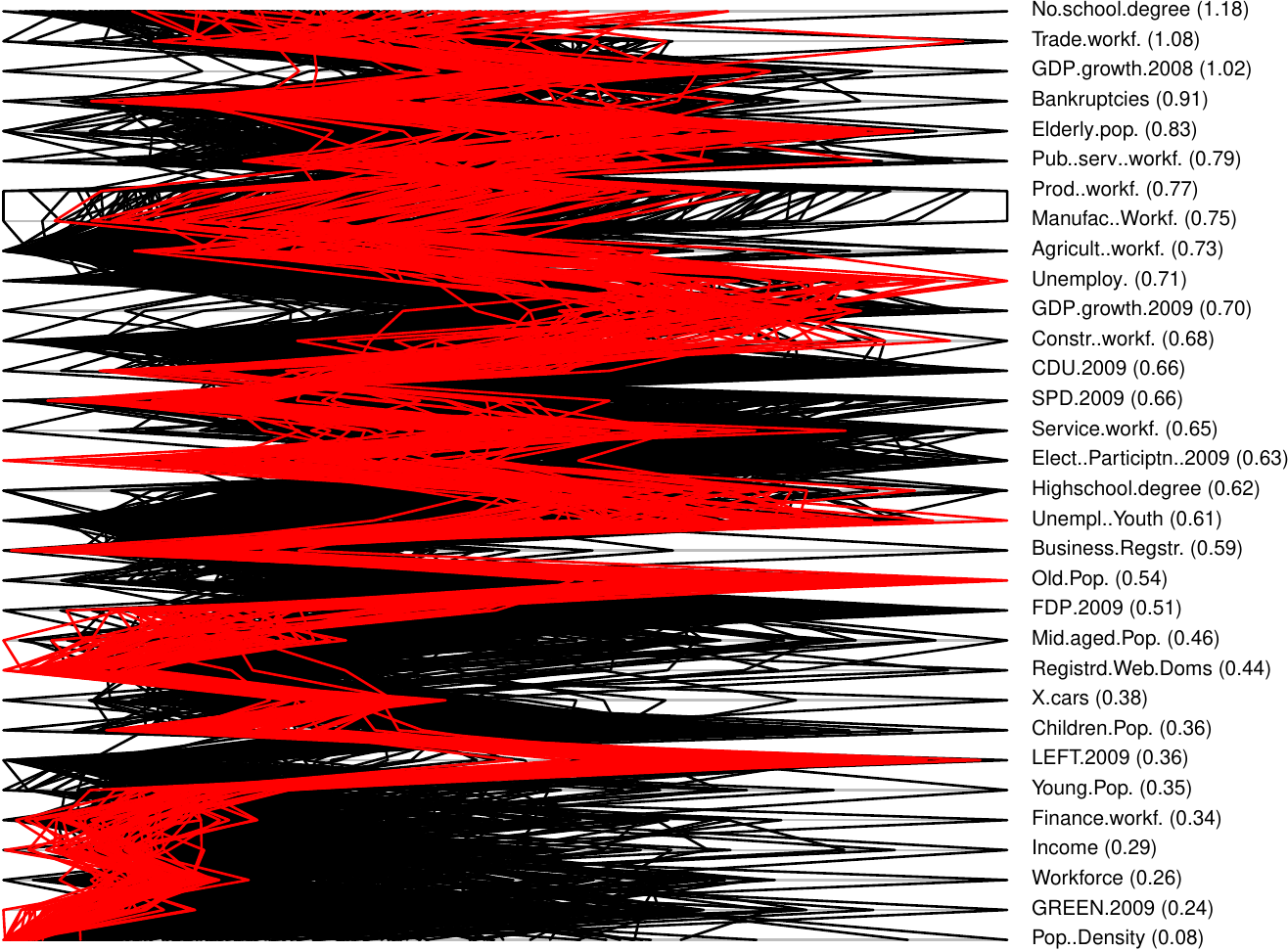}
  \caption{Parallel coordinates plot with Cluster 1 highlighted.}
  \label{fig:explore:pcp:c1}
  \end{subfigure}
  \begin{subfigure}[b]{0.45\textwidth}
    \includegraphics[width=0.8\columnwidth]{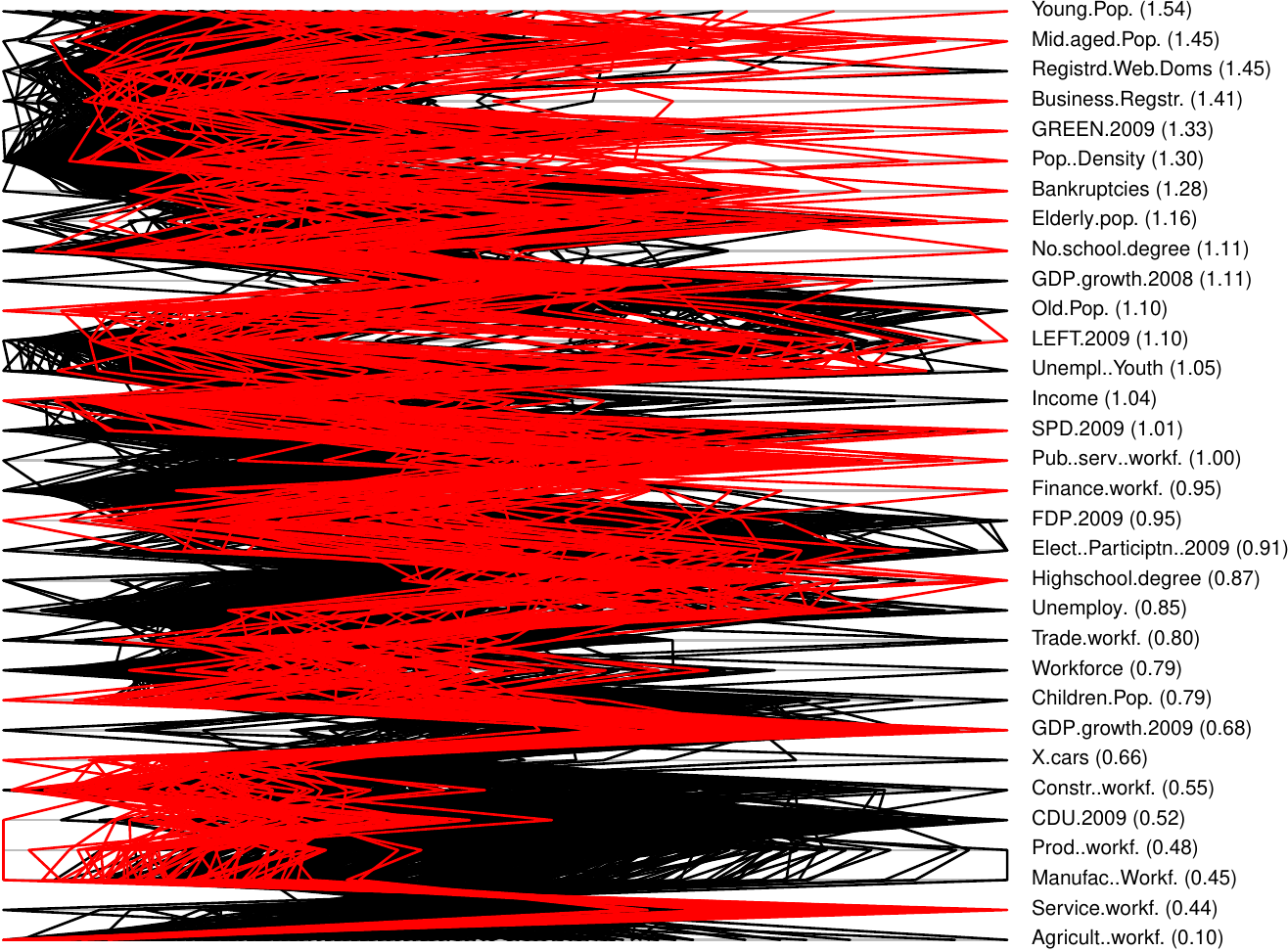}
    \caption{Parallel coordinates plot with Cluster 2 highlighted.}
  \label{fig:explore:pcp:c2}
  \end{subfigure}
  \caption{Parallel coordinates plots of the two clusters. Left: Cluster 1 highlighted; right: Cluster 2 highlighted.}
  \label{fig:explore:pcp}
 \end{figure*}

\begin{table}[!h]
\centering
\begin{footnotesize}
  \begin{tabular}{l l}
    \toprule
    \textbf{Group} & \textbf{Attributes} \\
\midrule
$C_1$  &  LEFT.2009, CDU.2009, SPD.2009,
\\ & FDP.2009, GREEN.2009 \\
&\\
$C_2$  & Elderly.pop., Old.Pop., Mid.aged.Pop., \\
& Young.Pop., Children.Pop.\\
&\\
$C_3$  & Agricult..workf., Prod..workf., \\
& Manufac..Workf., Constr..workf., \\
& Service.workf., Trade.workf., \\
& Finance.workf., Pub..serv..workf. \\
& \\
$C_4$  & Highschool.degree, No.school.degree, \\
& Unemploy., Unempl..Youth, Income\\
\bottomrule
  \end{tabular}
  \end{footnotesize}
  \caption{Column groups in the focus tile.}
  \label{tab:focus:cgs}
\end{table}

\section{Exploration of the \textsc{german} data with specific hypotheses}
\label{sec:german2}

\subsection*{Case 1: No background knowledge}
In this section we consider investigating hypotheses involving a
subset of data items (rows in the data matrix, corresponding to
different districts) and attributes.

We want to investigate a hypothesis concerning the relations between
certain attribute groups in {\em rural areas}. We hence define our
hypotheses as follows. As the subset of rows $R$ we choose all
districts that are of the rural type. We then partition a subset of
the attributes into four groups. The first attribute group ($C_1$)
consists of the voting results for the political parties in 2009. The
second attribute group ($C_2$) describes demographic properties such
as the fraction of elderly people, old people, middle aged people,
young people, and children in the population. The third group ($C_3$)
contains attributes describing the workforce in terms of the fraction
of the different professions such as agriculture, production, or
service. The fourth group ($C_4$) contains attributes describing the
education level, unemployment, and income. The attribute groupings are
listed in Table~\ref{tab:focus:cgs}.

\begin{figure*}[h]
\centering
  \vspace*{-4ex}
  \begin{subfigure}[b]{0.4\textwidth}
  \includegraphics[width=\textwidth]{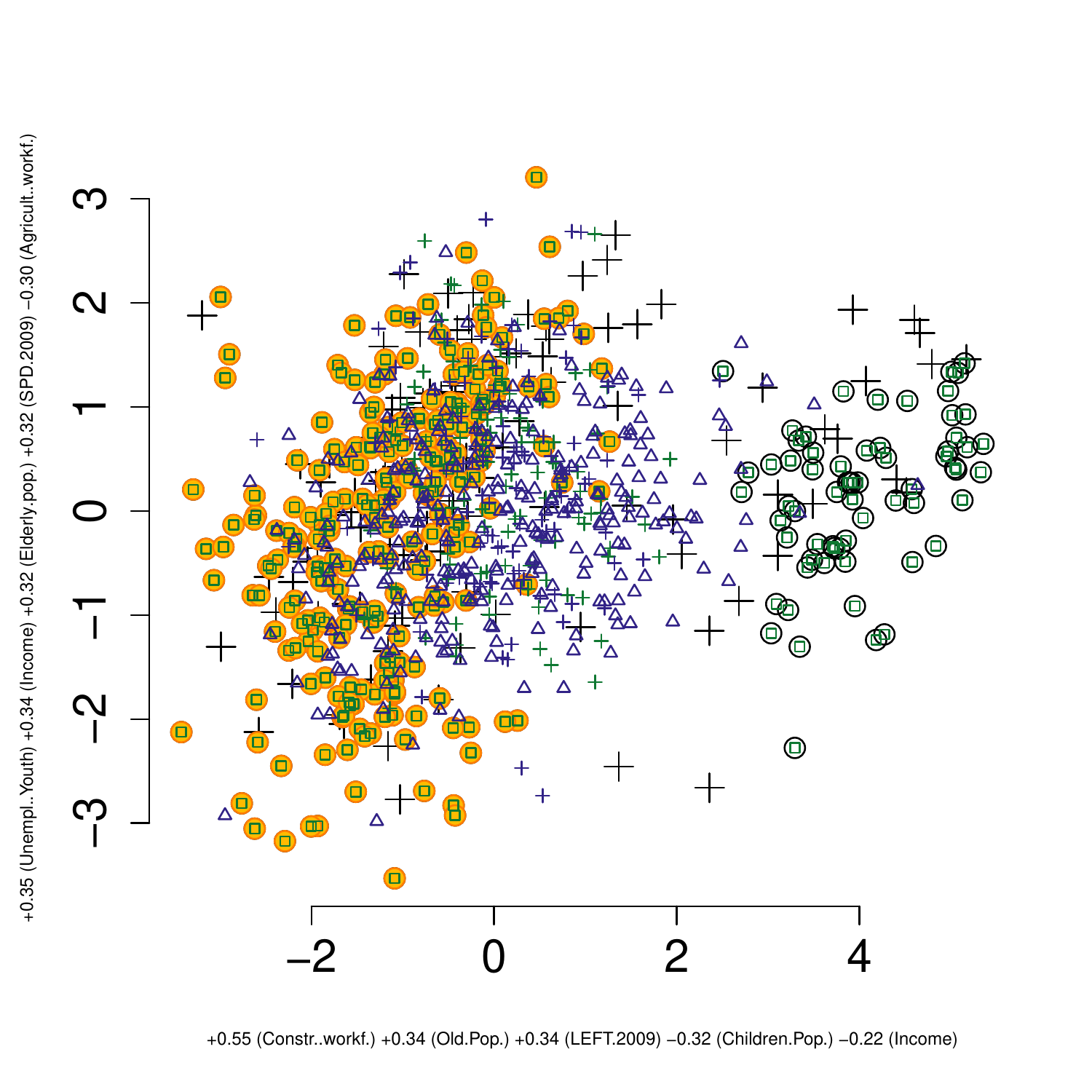}
  \caption{Cluster 3 marked}
  \label{fig:focus:c1}
  \end{subfigure}
  \begin{subfigure}[b]{0.4\textwidth}
  \includegraphics[width=\textwidth]{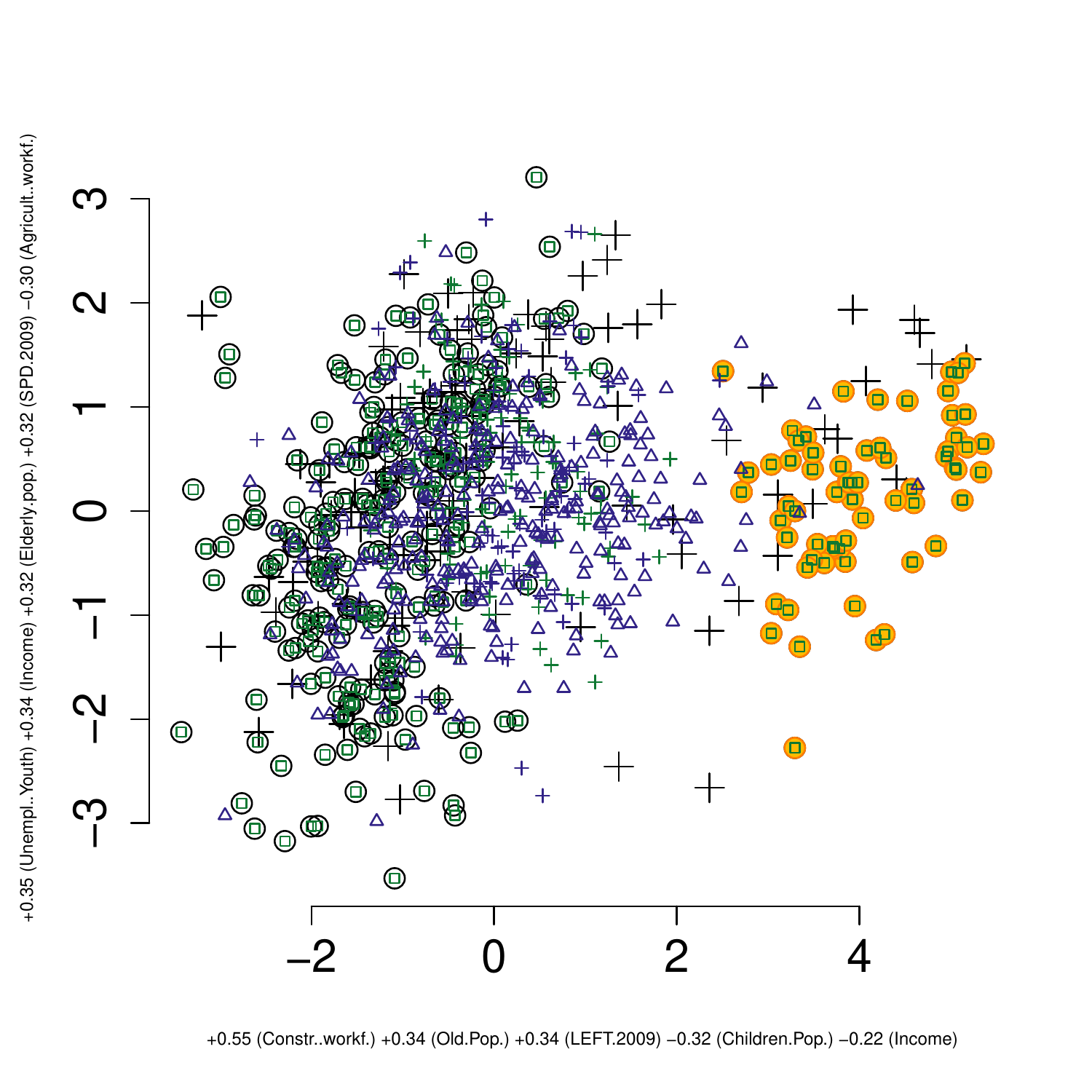}
  \caption{Cluster 4 marked}
  \label{fig:focus:c2}
  \end{subfigure}
    \begin{subfigure}[b]{0.4\textwidth}
    \end{subfigure}

    \begin{subfigure}[b]{0.45\textwidth}
  \includegraphics[width=0.8\columnwidth]{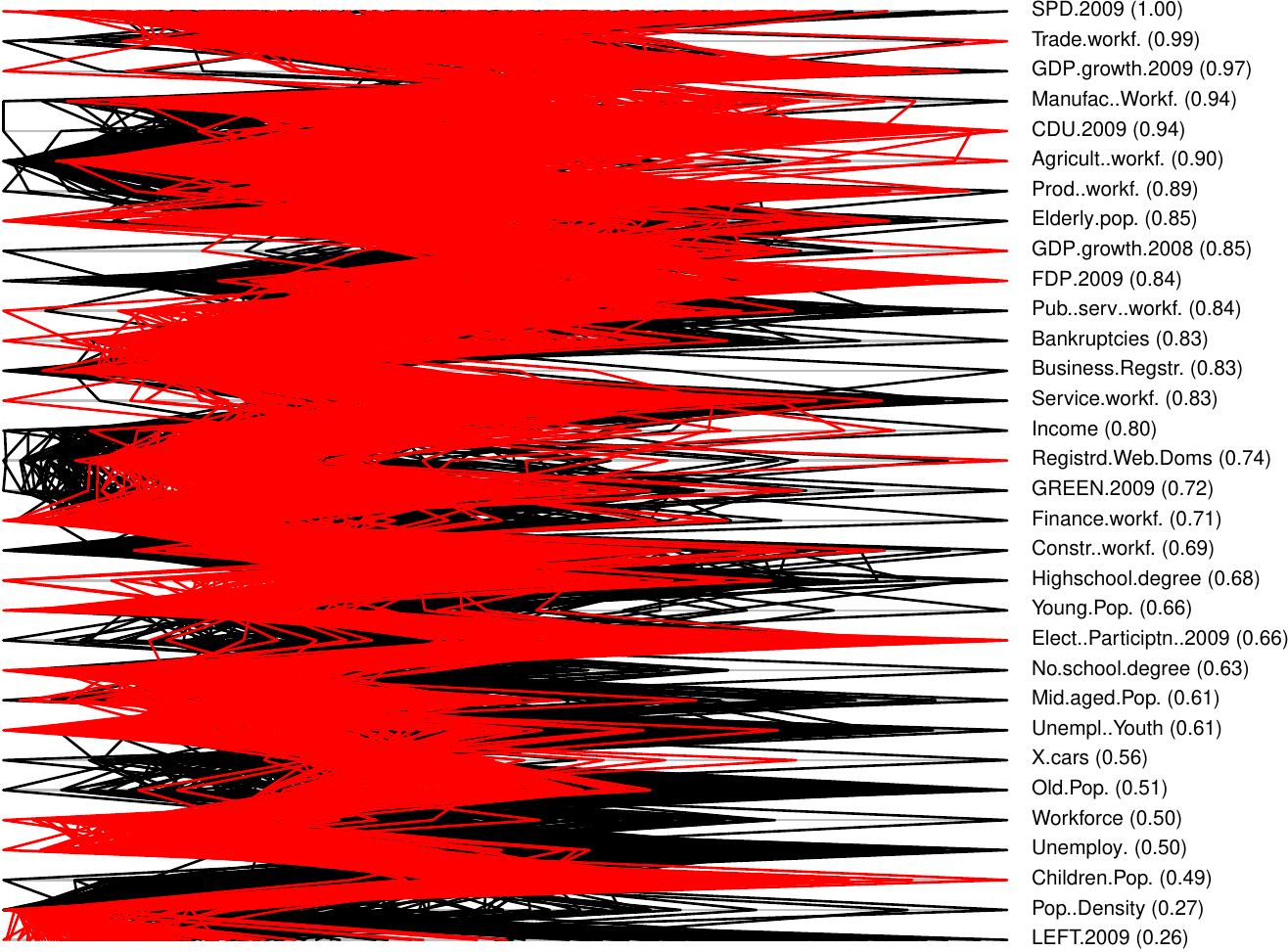}
  \caption{Parallel coordinates plot with Cluster 3 highlighted.}
  \label{fig:focus:pcp:fc1}
  \end{subfigure}
  \begin{subfigure}[b]{0.45\textwidth}
    \includegraphics[width=0.8\columnwidth]{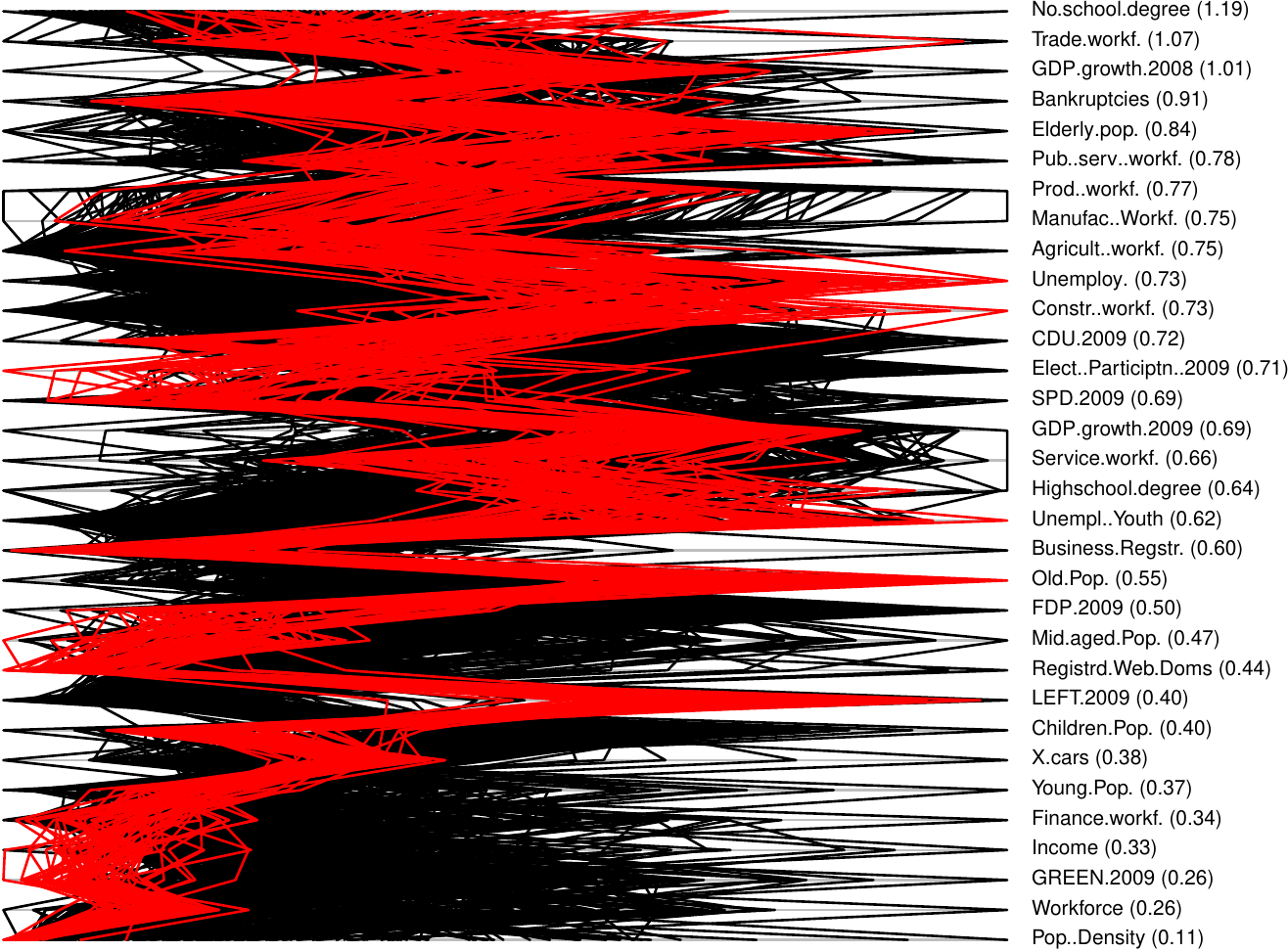}
    \caption{Parallel coordinates plot with Cluster 4 highlighted.}
  \label{fig:focus:pcp:fc2}
  \end{subfigure}
 \caption{ \small{ Initial view of the data. \textbf{Top row:} symbols as in
   Fig. \ref{fig:explore:vee}, except that points outside the focus area
   are shown using $+$ where black means data, and green and blue
   means samples from the first and second part of the hypothesis
   pair, respectively. Two visible clusters (3 and 4) are
   marked. \textbf{Bottom row:} parallel coordinates plots of the two
   clusters.  } }
\end{figure*}


We then form a hypothesis pair $\mathcal{H}_{F, \emptyset}=\langle
\emptyset+\mathcal{T}_{F_1}, \emptyset+\mathcal{T}_{F_2} \rangle$,
where $\mathcal{T}_{F_1}$ consists of a tile spanning all rows in $R$
and the columns $C=C_1\cup C_2\cup C_3\cup C_4$
whereas $\mathcal{T}_{F_2}$ consists of tiles $(R,
C_i)$, $i\in[4]$. These focus tiles allow us to investigate whether
 there are relations between these attribute groups,
while ignoring relations inside the attribute groups.

The view where the distributions characterised by the pair
$\mathcal{H}_{F, \emptyset}$ differ the most is shown in
Figure~\ref{fig:focus:c1}.  Again, we use green squares ($\square$) to
denote items belonging to the data sample corresponding to
$\mathcal{T}_u+\mathcal{T}_{F_1}$ and blue triangles ($\triangle$) to
denote items belonging to the data sample corresponding to
$\mathcal{T}_u+\mathcal{T}_{F_2}$. The points outside the focus area
are shown using a plus-sign ($+$). We notice a clear division into two
clusters, one on the left (marked in Figure~\ref{fig:focus:c1}) and
one on the right (marked in Figure~\ref{fig:focus:c2}). We now
investigate these two clusters in a similar fashion as we did before
when exploring the relations in the data. Based on the {\em Region}
and {\em Type} attributes for the clusters shown in
Table~\ref{tab:focus:clusters}, we conclude that Cluster 3 represents
rural districts in the North, South and West, whereas Cluster 4
represents rural districts in the East. Based on the parallel
coordinates plots in Figure~\ref{fig:focus:pcp:fc1} and
Figure~\ref{fig:focus:pcp:fc2} it is clear that the voting behaviour
is one aspect separating Cluster 3 and Cluster 4. In Cluster 4, the
support for the Left party is prominent. Also, the fraction of old
people in the population in Cluster 4 is larger, whereas the fraction
of children in the population is high in Cluster 3. We conclude that
there are interesting relations between the attribute groups
considered, which means, e.g., that there is a connection between
demographic properties and voting behaviour in the different rural
districts.

\begin{table}[h]
\centering
\begin{footnotesize}
\begin{subtable}{0.22\textwidth}
\begin{tabular}{ll}
  \toprule
  Region &    Type \\
  \midrule
East :  0   & Rural:233   \\
  North: 48   & Urban:  0   \\
  South:106   &  \\
  West : 79   &  \\
   \bottomrule
\end{tabular}
\caption{Cluster 3.}
  \label{tab:focus:c1}
\end{subtable}
\begin{subtable}{0.22\textwidth}
\begin{tabular}{ll}
  \toprule
  Region &    Type \\
  \midrule
East :64   & Rural:65   \\
  North: 0   & Urban: 0   \\
  South: 0   &  \\
  West : 1   &  \\
   \bottomrule
\end{tabular}
\caption{Cluster 4.}
  \label{tab:focus:c2}
\end{subtable}
\end{footnotesize}
\caption{{\em Region} and {\em Type} attributes for clusters observed in the data when focusing on a subset.}
  \label{tab:focus:clusters}
\end{table}

\subsection*{Case 2: Using background knowledge}
We could also have used our already observed background knowledge. Let
$t$ be a tile corresponding to Cluster 1 in
Table~\ref{tab:explore:c1}. We hence consider the hypothesis pair
$\mathcal{H}_{F, \{t\}}=\langle \{t\}+\mathcal{T}_{F_1},
\{t\}+\mathcal{T}_{F_2} \rangle$. Using these hypotheses we get the
view shown in Figure~\ref{fig:focus:ft}. This view is clearly
different
than Figure~\ref{fig:focus:c1}, since we
already were aware of the relations concerning the rural districts in the
East and this was included in our background knowledge. We hence conclude that
the background knowledge matters
when comparing hypotheses.

\begin{figure}[h]
    \centering
  \includegraphics[width=0.4\textwidth]{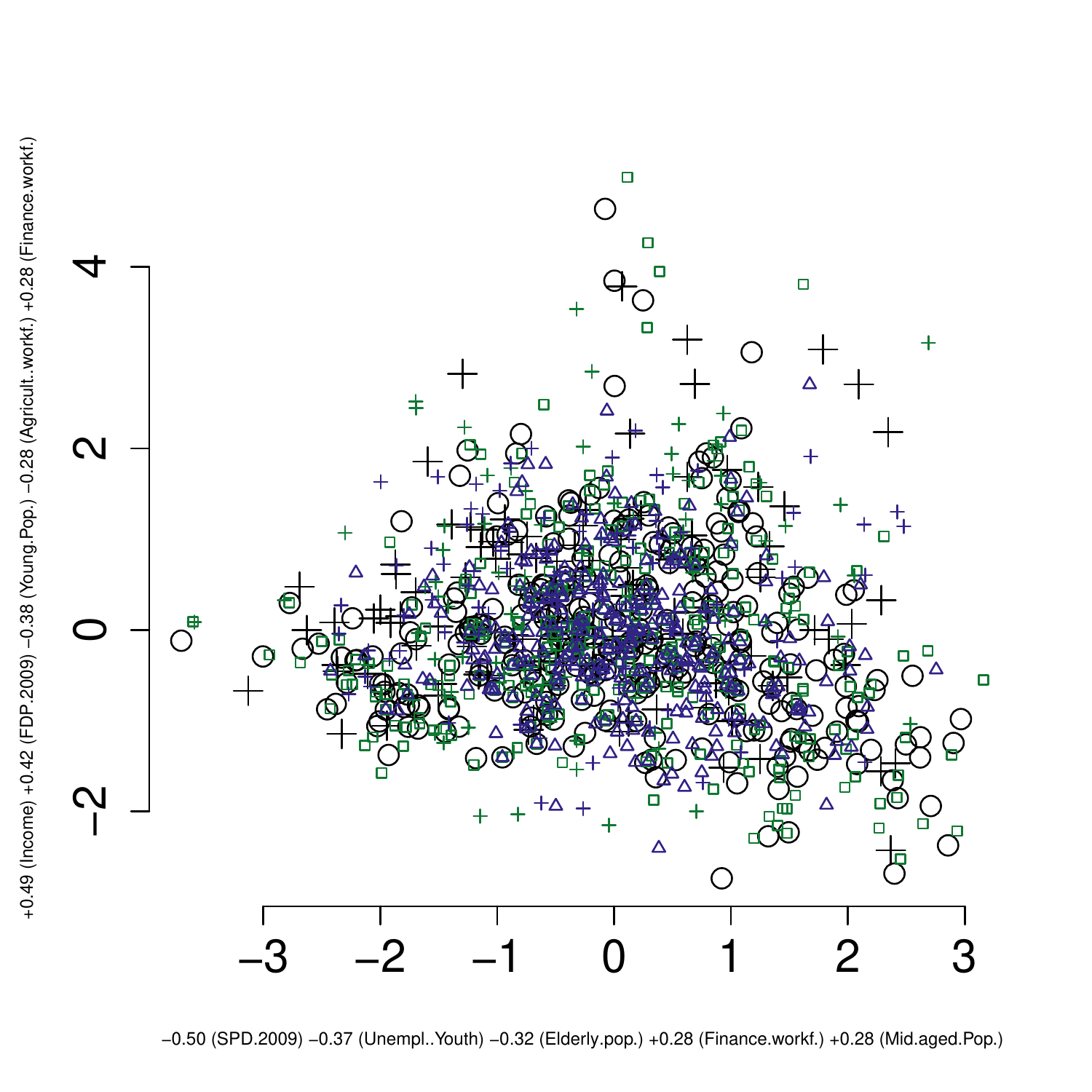}
  \caption{\small{ Initial view of the data using Cluster 1
      (Tab.~\ref{tab:explore:c1}) as background knowledge. Symbols as
      in Fig.~\ref{fig:focus:c1}.}}
  \label{fig:focus:ft}
\end{figure}